\documentclass{article}

 \usepackage[preprint]{neurips_2026}

\usepackage[utf8]{inputenc} 
\usepackage[T1]{fontenc}    
\usepackage{hyperref}       
\usepackage{url}            
\usepackage{booktabs}       
\usepackage{amsfonts}       
\usepackage{nicefrac}       
\usepackage{microtype}      
\usepackage{xcolor}         

\usepackage{amsmath}
\usepackage{amssymb}
\usepackage{amsthm}

\newtheorem{theorem}{Theorem}

\newtheorem{proposition}{Proposition}

\theoremstyle{definition}
\newtheorem{definition}{Definition}
\newtheorem{example}{Example}
\theoremstyle{remark}
\newtheorem{remark}{Remark}
\newtheorem{assumption}{Assumption}

\usepackage{rotating}
\usepackage{fancyvrb}
\usepackage{url}
\usepackage{booktabs}
\usepackage{nicefrac}
\usepackage{comment}
\usepackage{framed}
\usepackage{caption}
\usepackage{subcaption}
\usepackage{bbm}
\usepackage{array}
\usepackage[english]{babel}
\usepackage[autostyle, english=american]{csquotes}
\MakeOuterQuote{"}
\usepackage{enumitem}
\usepackage{algorithm}
\usepackage{algpseudocode}
\usepackage{multirow}
\usepackage{tikz}
\usepackage{mathtools}
\usetikzlibrary{arrows.meta, positioning}
\usepackage{enumitem}

\usepackage{listings}
\usepackage{xcolor}
\definecolor{codeblue}{RGB}{20,80,180}
\definecolor{codegreen}{RGB}{20,120,60}
\definecolor{codegray}{RGB}{90,90,90}
\definecolor{codepurple}{RGB}{130,50,160}

\lstdefinestyle{hpo}{
  language=Python,
  basicstyle=\ttfamily\footnotesize,
  keywordstyle=\color{codeblue},
  commentstyle=\color{codegreen},
  stringstyle=\color{codepurple},
  numbers=left,
  numberstyle=\tiny\color{codegray},
  stepnumber=1,
  numbersep=6pt,
  showstringspaces=false,
  breaklines=true,
  columns=fullflexible,
  frame=single,
  framerule=0.3pt,
  xleftmargin=1.5em,
  framexleftmargin=1.2em
}

\newcommand{\commented}[1]{}

\newcommand{\eg}{{\em e.g.\/}, }

\title{Policy Optimization in Hybrid Discrete-Continuous Action Spaces via Mixed Gradients}

\author{%
  Matias Alvo \\
  Graduate School of Business\\
  Columbia University\\
  \texttt{malvo26@gsb.columbia.edu} \\
  \And
  Daniel Russo \\
  Graduate School of Business\\
  Columbia University\\
  \texttt{djr2174@gsb.columbia.edu} \\
  \AND
  Yash Kanoria \\
  Graduate School of Business\\
  Columbia University\\
  \texttt{yk2577@gsb.columbia.edu} \\
}

\begin{document}

\maketitle

\begin{abstract}

We study reinforcement learning in hybrid discrete–continuous action spaces, such as settings where the discrete component selects a regime (or index) and the continuous component optimizes within it — a structure common in robotics, control, and operations problems. Standard model-free policy gradient methods rely on score-function (SF) estimators and suffer from severe credit-assignment issues in high-dimensional settings, leading to poor gradient quality. On the other hand, differentiable simulation largely sidesteps these issues by backpropagating through a simulator, but the presence of discrete actions or non-smooth dynamics yields biased or uninformative gradients.
To address this, we propose Hybrid Policy Optimization (HPO), which backpropagates through the simulator wherever smoothness permits, using a mixed gradient estimator that combines pathwise and SF gradients while maintaining unbiasedness. We also show how problems with action discontinuities can be reformulated in hybrid form, further broadening its applicability. Empirically, HPO substantially outperforms PPO on inventory control and switched linear-quadratic regulator problems, with performance gaps increasing as the continuous action dimension grows.
Finally, we characterize the structure of the mixed gradient, showing that its cross term — which captures how continuous actions influence future discrete decisions — becomes negligible near a discrete best response, thereby enabling approximate decentralized updates of the continuous and discrete components and reducing variance near optimality.
All resources are available at \url{https://github.com/MatiasAlvo/hybrid-rl}.

\end{abstract}

\section{Introduction}
\label{sec:introduction}

Reinforcement learning (RL) at scale faces a fundamental credit assignment problem. Methods such as REINFORCE \citep{williams1992simple} and Proximal Policy Optimization (PPO) \citep{schulman2017proximal} estimate policy gradients by correlating perturbations in the policy's actions with the resulting cumulative reward. When a trajectory contains many decisions, each itself a high-dimensional action vector, this estimation problem becomes severely underdetermined: a single scalar return must be attributed back across thousands of action components whose effects are entangled with one another and with exogenous randomness. The signal-to-noise ratio of score-function (SF) gradient estimators degrades accordingly, and learning becomes slow or stalls altogether.

Recent work shows that this difficulty can be largely sidestepped when given a simulator that exposes additional structure beyond what a generic MDP provides. Two properties are central: exogeneity and smoothness. Under exogeneity, the underlying randomness --- customer demand, weather, market movements --- is independent of the agent's actions, so a recorded scenario can be replayed against any policy. Although the decision-maker (DM) still acts under partial information, the simulator has access to the full realization, enabling offline evaluation of arbitrary policies on the same trajectory. Under smoothness, small changes in actions induce small, computable changes in costs and state transitions, allowing gradients to be obtained by backpropagation through the simulator. Rather than relying on noisy outcome correlations, pathwise (PW) gradients compute how each action affects total cost via the chain rule, thereby largely mitigating the credit assignment problem. This paradigm has driven recent progress in robotics \citep{hu2019difftaichi,levine2016end}, inventory management \citep{alvo2023deep,madeka2022deep}, and queueing control \citep{che2024differentiable}.

However, PW gradients are brittle to even mild departures from smoothness: discontinuities in the cost or transition functions introduce bias in the resulting gradients, and settings with purely discrete actions render PW gradients entirely uninformative. As soon as a problem fails to satisfy \textit{all} the conditions for clean PW estimation, the standard practice is to fall back to fully model-free methods — often discarding the very structure that made differentiable simulation effective in the first place.

In practice, many problems satisfy \textbf{most} conditions allowing the use of PW gradients. For example, several problems involve a \textit{hybrid} action space, in which some actions are discrete and others are continuous, and the model is smooth with respect to the continuous actions. One such example is the Switched Linear Quadratic Regulator (S-LQR) \citep{zhang2009value}, in which the DM selects a discrete \textit{mode} that defines the transition and cost functions, and a continuous control input given the mode. Additionally, some problems have only a ``few'' discontinuities. For instance, in the joint replenishment problem (JRP) with large fixed ordering costs \citep{goyal1989joint}, the cost function is discontinuous at zero but continuous elsewhere. In the settings outlined above, should we discard PW gradients entirely, or can we exploit them wherever structure permits?

\medskip
{\bf Main contributions.}
In this work, we develop policy gradient algorithms for settings involving hybrid discrete-continuous action spaces. We propose a modeling framework and structured policy classes that enable an unbiased gradient  estimator blending PW and SF gradients. Our design leverages PW gradients wherever possible, enabling sharper credit assignment where the problem structure allows. We call our framework \textit{Hybrid Policy Optimization} (HPO). As shown in Figure~\ref{fig:performance-results}, HPO scales far better with the continuous action dimension than PPO, the go-to policy-based RL algorithm, with especially dramatic gaps in the S-LQR. 
Full results in Appendix~\ref{appendix:results} confirm that HPO consistently outperforms PPO in high-dimensional settings. As we show in 
Section~\ref{sec:gradient-quality-exps} (see Figure~\ref{fig:ppo_vs_hybrid}), 
this performance gap is driven by HPO's \textit{mixed} gradient estimator 
--- which blends PW and SF terms --- providing a 
drastically better gradient signal than  pure SF-type gradients, which PPO relies on.

\begin{figure}
\centering
\begin{subfigure}{0.48\textwidth}
    \centering
    \includegraphics[width=\linewidth]{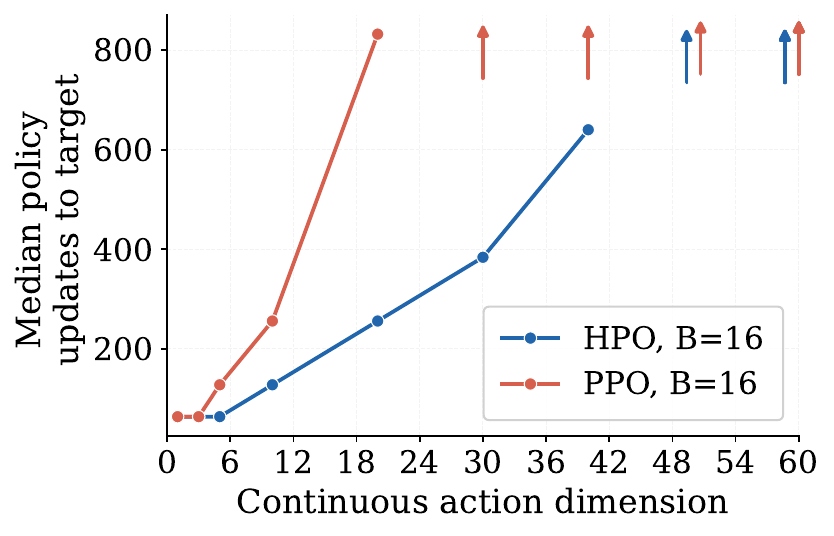}
    \caption{Joint replenishment problem.}
    \label{fig:performance-jrp}
\end{subfigure}
\hfill
\begin{subfigure}{0.48\textwidth}
    \centering
    \includegraphics[width=\linewidth]{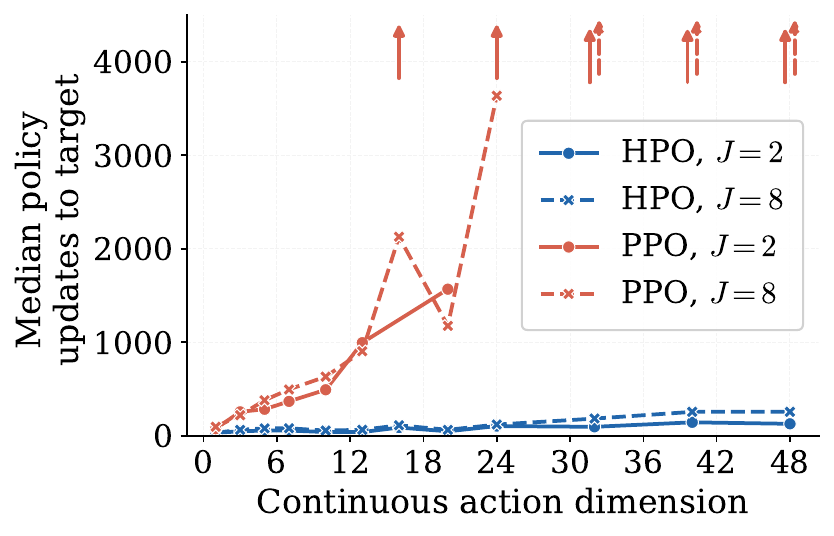}
    \caption{Switched LQR, two mode counts $J$.}
    \label{fig:performance-lqr}
\end{subfigure}
\caption{Median number of policy updates required for HPO and PPO to reach $10\%$ of the best observed run, as the continuous action dimension $p$ increases, on the JRP (left) and switched LQR (right). We consider 10 independent trials per configuration. Vertical arrows represent configurations for which the median does not reach the target within the allotted budget.}
\label{fig:performance-results}
\end{figure}

Our main contributions are as follows:
\begin{enumerate}[leftmargin=*]

    \item \textbf{Hybrid Policy Optimization framework:} We propose a policy optimization framework for settings with hybrid discrete–continuous action spaces that exploits known smoothness in the environment. Our structured ``towered'' policy representation—a stochastic discrete policy component paired with a deterministic continuous controller—yields an unbiased \textit{mixed} gradient estimator combining PW and SF terms. To extend this framework to problems with action discontinuities, we show how such problems can be expressed in hybrid form.
    
    \item \textbf{Empirical evaluation:} We find that HPO significantly outperforms PPO, with performance gains that increase with the continuous action dimension. We also find that the mixed gradient can yield much more accurate gradient estimation than pure SF estimators (see Figure~\ref{fig:ppo_vs_hybrid}), supporting the intuition that PW gradients improve credit assignment.
    
    \item {\bf Toward decentralized training:} The mixed gradient entangles the discrete and
continuous components through a ``cross term'', necessitating centralized computation
with privileged access to both. We prove this coupling vanishes near a discrete best
response (Theorem~\ref{thm:cross-term-vanishes-for-near-opt}), grounding approximate
decentralized updates of the two components. The vanishing also motivates dropping the cross term as a variance-reduction technique (Section~\ref{sec:cross-term-vanishes}).
    
\end{enumerate}

{\bf Overview of related work.}
Several works have tackled discrete--continuous action spaces using \emph{model-free RL}. 
\citet{xiong2018parametrized} introduced P-DQN, which combines DQN-style discrete action selection with deterministic policy-gradient updates for the continuous parameters. 
\citet{huang2022mixed} proposed a related DQL--DDPG approach, using an actor to generate continuous controls for each candidate discrete action and a shared critic to evaluate the resulting hybrid actions. 
\citet{li2021hyar} used a variational autoencoder to represent discrete--continuous hybrid actions in a continuous latent space, optimizing the policy with standard RL algorithms. 
These works treat the environment as a black box and therefore cannot exploit PW gradients, even when smoothness is available.

Our work builds on the literature of unbiased gradient estimation for stochastic objectives, a central topic in modern ML. We refer readers to \citet{mohamed2020monte} for an accessible introduction. Of particular importance is \citet{schulman2015gradient}, who propose a general recipe for computing unbiased gradient estimators in computational graphs with both stochastic and deterministic operations. We build on this framework and specialize it to hybrid action spaces, deriving a ``mixed'' gradient estimator tailored to the structure of these problems.
\citet{levy2018deterministic} derive mixed gradient estimators for problems with discrete actions, but do so by introducing a continuous relaxation of the dynamics and optimizing the resulting surrogate objective. We instead exploit the structure of hybrid action spaces directly: smoothness holds conditional on the discrete action, which suffices to obtain an unbiased estimator for the original objective without relaxing the dynamics.

Our work also broadens the applicability of differentiable simulation 
\citep{hu2019difftaichi, alvo2023deep, madeka2022deep}, which has primarily 
focused on fully smooth settings. Prior work has extended these methods 
to discrete decisions using surrogate gradients such as the straight-through 
estimator \citep{bengio2013estimating} (see, \eg \cite{che2024differentiable}). 
We complement these approaches by exploiting problem structure in hybrid 
discrete--continuous action spaces, leveraging PW gradients wherever 
differentiability is available in a manner that preserves unbiasedness.

{\bf Limited scope.}
Our work follows a line of research that assumes access at training time to an exogenous simulator, enabling replay of alternative policies under the same realization of uncertainty (Section ~\ref{sec:problem-description}); our methods therefore do not apply to fully black-box simulators. Within this setting, our central question is whether the benefits of differentiable simulation can be preserved in hybrid discrete–continuous action spaces. The natural ablation is to replace the mixed estimator with a pure SF estimator (PPO) over the entire policy, which is our primary point of comparison. Methods designed for hybrid action spaces in black-box settings (\eg Parameterized DQN, HyAR) address complementary problems (like representing hybrid actions) without exploiting simulator differentiability — and still rely on SF-style credit assignment for optimization. By contrast, our goal is to preserve the stronger credit assignment enabled by direct backpropagation through smooth components of the simulator wherever possible. Finally, we focus on benchmarks (JRP and  S-LQR) that allow systematic variation of the continuous action dimension, isolating its effect on gradient quality and optimization behavior. A natural next step is to apply this framework to large-scale problems in multi-echelon inventory control, queueing control, and robotic control, where differentiable simulators are increasingly available but discrete decisions have limited the reach of PW methods.

\section{Problem Description}\label{sec:problem-description}

We consider a setting recently dubbed an \emph{exogenous MDP} — less a property of the underlying control problem than a description of the \emph{simulator available at training time}. Classical dynamic programming assumes the transition model is analytically tractable, while standard RL treats the environment as a black box with no ability to reuse randomness across policies. An exogenous MDP sits between these extremes: the simulator evaluates policies on a fixed \emph{scenario} (a realization of the underlying exogenous randomness) using known dynamics, enabling direct policy comparison on the same trajectory and sharper credit assignment. \emph{Crucially, the DM does not observe the scenario at decision time — the simulator's privileged access is a training-time artifact, not a relaxation of the information structure.} This structure arises naturally in robotics (scenarios fix physical parameters; the agent observes state and acts), logistics (known flow or service dynamics), and inventory control (scenarios are demand realizations). We formalize this setting as follows.

\subsection{Markov decision processes with exogenous randomness \label{sec:mdp} }

There are two common (and equivalent) mathematical abstractions of an MDP: one specified by a transition kernel $P(\cdot\mid s,a)$, and one written in ``disturbance'' form as a deterministic transition map $f$ driven by exogenous noise $\xi_t$ \citep{bertsekas2012dynamic}. 
We adopt the latter formalism, since it streamlines the development of our gradient estimators and algorithms.

In each period $t$, a DM observes state $s_t\in\mathcal{S}$, chooses action $a_t\in\mathcal{A}(s_t)$, and the system evolves as $s_{t+1}=f(s_t,a_t,\xi_t)$ with cost $c(s_t,a_t,\xi_t)$, where $\xi=(\xi_0,\ldots,\xi_{T-1})$ is a sequence of exogenous random terms with joint law $\mathbb{P}_{\xi}$. We denote an MDP by $\mathcal{M}=(\mathcal{S},\mathcal{A},c,f,\mathbb{P}_{\xi})$ and
define the expected cost $J(\pi)$ for initial state $s_0$, which the DM seeks to minimize over $\pi\in\Pi$:
\[
J(\pi) := \mathbb{E}_{\pi}\!\left[\sum_{t=0}^{T-1} \gamma^t\, c(s_t,\pi_t(s_t),\xi_t)\,\Big|\,s_0\right].
\]
We next impose a substantive assumption on the training-time simulator. Following \textit{hindsight RL} \citep{alvo2023deep,xie2024vc} and \textit{exogenous MDPs} \citep{sinclair2022hindsight}, we assume that the simulator has access to the transition and cost functions $(f,c)$ and observes the realized exogenous disturbances. Therefore, it can evaluate any policy counterfactually on the same fixed scenario $\bar{\xi}$.

\subsection{Hybrid discrete-continuous action MDPs}
\label{sec:hybrid-mdps}

Our methods apply to MDPs with hybrid discrete--continuous action spaces, which we call 
\textit{hybrid MDPs}. The state $s$ is a continuous vector and the action is a pair 
$a=(x,b)$, where $x$ is a discrete action selected from a finite set $\mathcal{X}$ and 
$b$ is a continuous action; the set of feasible continuous actions $\mathcal{B}^x$ may depend on 
$x$. Our primary assumption is that the transition and cost functions are smooth with 
respect to $b$ and $s$, but not $x$: that is, $f(s,(x,b),\xi)$ and $c(s,(x,b),\xi)$ 
are smooth with respect to $(s,b)$ for any fixed $x$ and $\xi$. Formal conditions are 
provided in Appendix~\ref{appendix:hybrid-mdps}.

We illustrate the framework with two examples that serve as our main benchmarks.

\begin{example}[Joint replenishment problem]\label{ex:jrp}
A retailer manages $p$ products over $T$ periods. At each period $t\in\{0,\ldots,T-1\}$, 
the retailer observes inventory $I_t=(I_t^1,\ldots,I_t^p)\in\mathbb{R}^p$ and outstanding 
orders $Q_t\in\mathbb{R}_+^{p\times(L-1)}$, where $L\geq 2$ is the lead time, giving state 
$s_t=(I_t,Q_t)$. The discrete action $x_t\in\{0,1\}$ indicates whether to place a shipment 
and the continuous action $b_t\in\mathbb{R}_+^p$ specifies order quantities; the 
executed order is $x_t\cdot b_t$. Demand $\xi_t\in\mathbb{R}_+^p$ is realized and the 
system updates as
\begin{equation*}
    I_{t+1}^k = I_t^k - \xi_t^k + Q_t^k[1], \qquad 
    Q_{t+1}^k = \bigl(Q_t^k[2],\ldots,Q_t^k[L-1],\,x_t b_t^k\bigr).
\end{equation*}
The per-period cost combines linear holding and underage costs with a fixed ordering cost $K$:
\begin{equation*}
    c(s_t,a_t,\xi_t) = \sum_{k=1}^p\bigl[u_k(\xi_t^k-I_t^k)^+ + h_k(I_t^k-\xi_t^k)^+\bigr] 
    + K\cdot x_t.
\end{equation*}
The exogenous scenario consists of $s_0$ and the demand sequence $\xi_{0:T-1}$.
\end{example}

\begin{example}[Switched LQR]\label{ex:slqr}
A generalization of the classical LQR problem \citep{kalman1960contributions}. At each period 
$t\in\{0,\ldots,T-1\}$, the agent observes state $s_t\in\mathbb{R}^p$, selects a discrete 
mode $x_t\in\{1,\ldots,J\}$ and a continuous control $b_t\in\mathbb{R}^p$. Each mode $j$ 
has dynamics matrices $(A_j,B_j)$ and positive-definite cost matrices $(Q_j,R_j)$, and 
the system evolves as
\begin{equation*}
    s_{t+1} = A_{x_t}s_t + B_{x_t}b_t + W_t, \qquad 
    c(s_t,a_t) = s_t^\top Q_{x_t}s_t + b_t^\top R_{x_t}b_t,
\end{equation*}
where $W_t$ is exogenous process noise. The exogenous scenario consists of $s_0$ and 
$W_{0:T-1}$. 
\end{example}

\begin{remark}[Piecewise-smooth MDPs (PS-MDPs)]
The hybrid MDP framework also encompasses problems in which the cost or transition functions are discontinuous in the action but smooth within finitely many regions of the action space, which we refer to as PS-MDPs. Such problems can be recast as hybrid MDPs by letting the discrete action select the smooth region and the continuous action specify the control within it. A formal definition and equivalence result are provided in Appendix~\ref{appendix:ps-mdps}.
\end{remark}

\section{Policy approximation and gradient computation \label{sec:policy-and-gradient}}

{\bf Policy representation.}
A policy in a hybrid MDP maps states to a joint distribution over discrete and continuous 
actions. Although these components generally do not decouple --- the continuous control 
should depend on the selected mode, as in the S-LQR --- we find it useful to factor the 
policy as $\pi = (\pi^{\mathcal{X}}, \pi^{\mathcal{B}})$, where $\pi^{\mathcal{X}}$ 
outputs a distribution over discrete actions and $\pi^{\mathcal{B}}$ deterministically 
outputs one action \emph{candidate} per discrete value $j \in \mathcal{X}$. The DM samples $x_t \sim \pi^{\mathcal{X}}(s_t)$ and applies the candidate $\pi^{\mathcal{B}}(s_t)[x_t]$ (see Figure~\ref{fig:hybrid-policy-architecture} 
in Appendix~\ref{appendix:policy-parameterization}).
We consider a parameterized hybrid policy $\pi_\theta$ with $\theta = (\phi, \kappa)$, 
where $\pi^{\mathcal{X}}_\phi$ and $\pi^{\mathcal{B}}_\kappa$ are parameterized 
independently. Parameter sharing is possible but would obscure the interplay between 
the two components central to our analysis (Section~\ref{sec:cross-term-vanishes}).

{\bf Pathwise gradient estimation.}
Suppose the discrete policy were fixed at $x_t = 1$ for all $t$. Then, for a fixed 
scenario $\xi$, the system evolves as the smooth deterministic dynamical system 
$s_{t+1} = f(s_t, (1, \pi^{\mathcal{B}}_{\kappa}(s_t)[1]), \xi_t)$, and the state 
$s_t(\kappa)$ is a smooth function of $\kappa$ computed by iterating $f$. The total cost 
$\sum_{t=0}^{T-1} c(s_t(\kappa), (1, \pi^{\mathcal{B}}_{\kappa}(s_t(\kappa))[1]), \xi_t)$ 
is also smooth in $\kappa$. Its gradient with respect to $\kappa$ --- obtained by 
direct backpropagation through the simulator --- is called the \emph{pathwise (PW) gradient}. Unlike SF, PW gradients isolate each action coordinate's effect on cost, sharpening credit assignment. 

{\bf Deriving the mixed policy gradient estimator.}
Towards clean mathematical formalizations including Theorem~\ref{thm:mixed-gradient} below, we adopt a discounted infinite-horizon formulation with $\gamma \in (0,1)$ and stationary policies (our experiments will operate on finite-horizon trajectories of length $T$). Standard regularity of dynamics and policies---smoothness of $f, c$ in $(s,b)$; differentiability and bounded log-gradients of the policy components; bounded costs---is formalized as Assumptions~\ref{assumption:smoothness-for-continuous} and~\ref{ass:policy-regularity} in Appendix~\ref{appendix:notation-and-assumptions}.

When the discrete component is non-trivial and stochastic, the scenario $\xi$ alone no
longer determines a smooth trajectory: the sampled $x_{0:\infty}$ also varies. The key
insight is that the within-trajectory smoothness established above still holds conditioned
on any realization of $(x_{0:\infty}, \xi)$. We give a formula for the policy gradient that
exploits this; see Algorithm~\ref{alg:hpo-loss} for PyTorch-style pseudocode.

The estimator is most naturally expressed via \texttt{stop\_grad} operators in
autodifferentiation frameworks like PyTorch or JAX. For the formal statement, fix an
arbitrary scenario $\xi$ and discrete sequence $x_{0:\infty}$. Conditional on this
realization, the state $s_t(\kappa)$ and continuous action $b_t(\kappa)$ are deterministic
smooth functions of $\kappa$, just as in the purely continuous case above. The theorem
statement suppresses the dependence on $(\xi, x_{0:\infty})$, focusing on the dependence of
the realized trajectory on $\kappa$. We write $\mathbb{E}_\theta$ for expectation under the joint law of $(\xi_{0:\infty}, x_{0:\infty})$ induced by $\pi_\theta$, set $J(\theta) := \mathbb{E}_\theta\bigl[\sum_{t=0}^{\infty}\gamma^{t} c(s_t(\kappa),(x_t,b_t(\kappa)),\xi_t)\mid s_0\bigr]$, and let $Q_{\pi_\theta}$ denote its action-value function.

\begin{theorem}[Mixed policy gradient estimator]
\label{thm:mixed-gradient}
    Under Assumptions \ref{assumption:smoothness-for-continuous} and \ref{ass:policy-regularity}, the objective
    $J(\theta)$ is differentiable in $\theta = (\phi,\kappa)$ and its gradients are given by:
\begin{align}
    \nabla_\phi J(\theta)
    &= \mathbb{E}_{\theta}\!\left[
    \sum_{t=0}^{\infty}\gamma^{t}
    \nabla_{\phi}\log\pi^{\mathcal{X}}_{\phi}\left(x_t\mid s_t(\kappa)\right)
    Q_{\pi_{\theta}}(s_t(\kappa),a_t)
    \right],
    \label{eq:grad-phi}
    \\[2pt]
    \nabla_\kappa J(\theta)
    &= \underbrace{\mathbb{E}_{\theta}\!\left[
    \sum_{t=0}^{\infty}\gamma^{t}\nabla_{\kappa}c\bigl(s_t(\kappa),(x_t,b_t(\kappa)),\xi_t\bigr)
    \right]}_{\text{PW gradient term}}+\underbrace{\mathbb{E}_{\theta}\!\left[
    \sum_{t=0}^{\infty}\gamma^{t}
    \nabla_{\kappa}\log\pi^{\mathcal{X}}_{\phi}\!\left(x_t\mid s_t(\kappa)\right)
    Q_{\pi_{\theta}}(s_t(\kappa),a_t)
    \right]}_{\text{Cross term}}
    \label{eq:grad-kappa}
\end{align}
    \end{theorem}

The gradient with respect to $\phi$ is the standard SF estimator. The  \emph{pathwise term} captures the direct effect of $\kappa$ on costs and the state trajectory, holding discrete actions fixed. The \emph{cross term} captures how  changes in $\kappa$ alter $s_t(\kappa)$, thereby shifting the probabilities the discrete policy assigns to each action. Section~\ref{sec:cross-term-vanishes} studies the effect 
of dropping this term.

    \begin{lstlisting}[style=hpo,caption={Hybrid policy optimization gradient estimation.},label={alg:hpo-loss}]
        def HPO_loss(policy, value_fn, xi, s0, gamma, drop_cross=False):
            s, loss = s0, 0
        
            for t in range(len(xi)):
                # If drop_cross=True, detach s only inside the discrete score:
                # this removes the kappa cross term but keeps the phi gradient.
                s_score = stop_grad(s) if drop_cross else s
        
                logits = policy.discrete(s_score)
                x = Categorical(logits).sample()      # discrete action, treated as fixed
                logp = log_prob(logits, x)            # autograd flows through logits, not x
        
                b = policy.continuous(s)[x]           # autograd flows through selected candidate
                cost = c(s, x, b, xi[t])              # smooth for fixed x
                s_next = f(s, x, b, xi[t])            # smooth for fixed x
        
                # Illustrating with a TD-style advantage; this weight is not differentiated.
                advantage = stop_grad(cost + gamma * value_fn(s_next) - value_fn(s))
                loss += gamma**t * cost               # pathwise term for kappa
                loss += gamma**t * advantage * logp   # phi score; kappa cross unless detached
    
                s = s_next
            return loss
        
        loss = HPO_loss(policy, value_fn, xi, s0, gamma, drop_cross=False)
        loss.backward()  # phi: score term; kappa: pathwise + cross term
        # use drop_cross=True to remove only the kappa cross term
        \end{lstlisting}

        In the pseudocode, \texttt{stop\_grad} denotes the standard \texttt{detach} operation 
in PyTorch. The sampled mode \texttt{x} is always treated as a fixed integer, so no 
derivative is taken through the sampling operation or through the index used to select 
\texttt{b = policy.continuous(s)[x]}. The flag \texttt{drop\_cross} detaches the state 
only inside the discrete log-probability, removing the $\kappa$-gradient through 
$\log\pi^{\mathcal{X}}_\phi(x_t\mid s_t(\kappa))$ while leaving the PW 
gradient for $\kappa$ and the SF gradient for $\phi$ unchanged. We use \textbf{HPOFull}  and \textbf{HPONoCross} to denote HPO  with and without the cross term, respectively.

\section{Numerical experiments}
\label{sec:numerical-exps}

{\bf Problem settings.} We empirically evaluate HPO on the two problem classes introduced in 
Section~\ref{sec:hybrid-mdps}, demonstrating performance gains over PPO as the continuous 
action dimension increases. We consider the JRP (Example~\ref{ex:jrp}) 
\citep{goyal1989joint}, a classical problem in operations management where fixed costs 
create interdependencies across products that make scaling fundamentally difficult 
\citep{ata2025computational}, and where PPO has been shown to underperform simple 
heuristics \citep{vanvuchelen2023use}. We also consider the S-LQR 
(Example~\ref{ex:slqr}) \citep{zhang2009value}, which extends the classical LQR 
\citep{kalman1960contributions} --- a canonical control benchmark --- to settings with discrete operating modes; finding the 
optimal mode sequence is NP-hard \citep{zhang2011infinite}.
 
{\bf Policy optimization.} Throughout our experiments, we consider \textbf{HPOFull},
and refer to it simply as HPO. As a baseline, we use PPO, following the CleanRL reference implementation~\citep{huang2022cleanrl}.
For consistency, our implementation of HPO also uses standard PPO to optimize the discrete policy
component. When computing $\nabla_\kappa J$, we substitute the GAE-based advantage estimates  produced by PPO into the SF loss term of Listing~\ref{alg:hpo-loss}. We train across
batches of scenarios, updating the continuous parameters with one full-batch policy update using the
mixed estimator, while the discrete parameters are updated using $5$ PPO epochs with $4$
minibatches per epoch. Full hyperparameters and implementation details are provided in
Appendix~\ref{appendix:implementation-details}. Code is available at
\url{https://github.com/MatiasAlvo/hybrid-rl}.
 
\subsection{Finding 1: HPO converges faster as the continuous action dimension grows}
\label{sec:performance-results}
 
We compare HPO and PPO on the JRP (Example~\ref{ex:jrp}) and the
 S-LQR (Example~\ref{ex:slqr}), varying the continuous action dimension $p$ across
10 independent trials per configuration. Our primary metric is the median number of policy
updates required to reach within $10\%$ of the best validation loss achieved by any algorithm
at each value of $p$; test loss across configurations is reported in Appendix~\ref{appendix:results}.

{\bf JRP.}
Figure~\ref{fig:performance-jrp} reports the median number of policy updates to reach a $10\%$
target gap, using batch size $B=16$ and trajectories of 100 periods, for $p$ ranging from $1$
to $60$. The two algorithms perform comparably at small dimensions, but their scaling behavior
diverges as $p$ grows: PPO's required policy updates rise steeply with the continuous action dimension,
while HPO's remains substantially flatter, though HPO with $B=16$ also fails to converge
within the training budget for $p \geq 50$. Figure~\ref{fig:scaling-jrp-targets} in Appendix \ref{appendix:jrp} extends this
comparison to batch sizes $B \in \{16, 64\}$ and a range of target gaps, revealing that for
$p \geq 40$, HPO with $B=16$ converges faster than PPO with $B=64$ at easier target gaps. We hypothesize this
stems from the superior quality of the mixed gradient estimator over the SF-type gradients
used by PPO, which we analyze directly in Section~\ref{sec:gradient-quality-exps}. Tables~\ref{tab:jrp-results-16}
and~\ref{tab:jrp-results-64} further corroborate these findings via test loss: HPO with $B=16$
achieves lower median test loss than PPO with $B=64$ for $p \geq 50$. Finally,
Figure~\ref{fig:scaling-jrp-one-off} in Appendix~\ref{appendix:ppo-hyperparams} shows that
hyperparameter tuning does not close the gap, suggesting the degradation is a fundamental
limitation of SF-based optimization.

{\bf S-LQR.}
Figure~\ref{fig:performance-lqr} reports the same metric for $p$ ranging from $1$ to $48$,
mode counts $J \in \{2, 8\}$, and batch size $128$. The performance gap is drastic. For 
both values of $J$, PPO degrades sharply once $p$ exceeds roughly $10$, with the median 
failing to converge for $p > 24$. Figure~\ref{fig:scaling-lqr-targets} in 
Appendix~\ref{appendix:lqr-results} further shows that PPO with $J = 2$ fails to reach 
even a $30\%$ target gap for $p>20$. HPO, by contrast, reaches a $5\%$ target across 
all configurations including $p = 48$, within 500 policy updates.
 
In Appendix~\ref{appendix:lqr-performance-validation}, we validate HPO against a 
Riccati-based baseline. For $J=1$ (standard LQR), where the Riccati solution is  optimal, HPO recovers performance within $0.2\%$ at $p=48$ 
(Figure~\ref{fig:hpo-riccati-j1}). For $J=4$, HPO improves upon the best single-mode 
Riccati baseline by up to $13\%$ (Figure~\ref{fig:hpo-riccati-j4}), demonstrating 
that it learns to exploit mode switching. Full results are in 
Appendix~\ref{appendix:lqr-results}.
 
\subsection{Finding 2: HPO yields superior gradient estimates as the action dimension grows}
\label{sec:gradient-quality-exps}
 
We now investigate \emph{why} HPO converges faster than PPO. The key claim is that 
backpropagating through the simulator sharpens credit assignment, increasingly so as 
the continuous action dimension grows. We test this by comparing the mixed and SF 
estimators of $\nabla_\kappa J(\theta)$ on the JRP.
 
{\bf Experimental setup.}
From Eq.~\eqref{eq:grad-kappa}, the mixed gradient combines a PW term with a cross term involving gradients of the discrete policy's log-likelihoods. As $\pi^{\mathcal{X}}_\phi$ becomes more peaked --- as it does when the discrete policy approaches a best response --- these log-likelihood gradients grow large near the boundaries between discrete actions, potentially inflating the cross term's variance. The signal quality of the mixed estimator may therefore vary substantially with the quality of the discrete component. We evaluate this along a natural training path in which the continuous policy is held fixed while the discrete policy improves, isolating the effect of discrete-policy quality from that of the continuous-action dimension $p$. Concretely, we train a policy to near-optimality, then make the continuous mapping stochastic by adding Gaussian noise so that the SF estimator for $\kappa$ is well-defined, and apply a fixed offset to the continuous outputs so that $\nabla_\kappa J$ is nonzero. We then re-initialize the discrete network and train it to convergence, saving checkpoints bucketed by performance gap, and report results stratified by this gap. Full details are provided in Appendix~\ref{appendix:gradient-alignment}.
 
{\bf Estimators.}
To apply the \emph{mixed} estimator, we use the reparameterization trick: under
$b_t = \mu_\kappa(s_t) + \sigma \varepsilon_t$ with $\varepsilon_t \sim \mathcal{N}(0, I)$,
the policy noise is treated as additional exogenous randomness absorbed into the scenario.
Conditioned on $(x_{0:T-1}, \xi_{0:T-1}, \varepsilon_{0:T-1})$, the trajectory is deterministic and
differentiable in $\kappa$, so the mixed-gradient derivation applies verbatim. For SF,
we consider a learned critic baseline to reduce variance; we omit PPO-style modifications, which bias the estimate.
 
{\bf Findings.}
Figure~\ref{fig:ppo_vs_hybrid} reports two metrics, each estimated from pairs of independent batch-level gradient estimates using batches of $512$ trajectories. \textbf{Alignment} is the expected cosine similarity between two independent estimates, measuring directional consistency. \textbf{RMSE} is the total standard deviation of the estimator across all coordinates, which can be estimated from pairs of independent estimates without access to the true gradient (see Appendix~\ref{appendix:gradient-alignment} for details).

SF exhibits alignment above $0.05$ only for small problem sizes ($p=1,3$), and quickly collapses to near zero as the dimension grows. This behavior is consistent across performance gaps and indicates that SF provides little directional signal in high-dimensional settings. Similarly, its RMSE increases rapidly with $p$, which contributes to the low alignment values. The deterioration in both metrics closely mirrors PPO's performance degradation observed in Figure~\ref{fig:performance-results}, suggesting that the gradients used by PPO suffer from poor signal quality in high-dimensional settings.

In contrast, for performance gaps of $30$--$35\%$, the mixed estimator maintains alignment 
above $0.8$ across all values of $p$, along with consistently low RMSE, indicating a stable 
gradient direction even in high-dimensional regimes. Both metrics degrade for smaller gaps 
($5$--$10\%$), but remain significantly better than those of SF. We hypothesize that this 
degradation is driven by increased variance in the cross term as the discrete policy approaches 
a best response, which we analyze in Appendix~\ref{appendix:gradient-quality-results}.
 
\begin{figure}
\centering
\begin{subfigure}{0.46\textwidth}
  \centering
  \includegraphics[width=\linewidth]{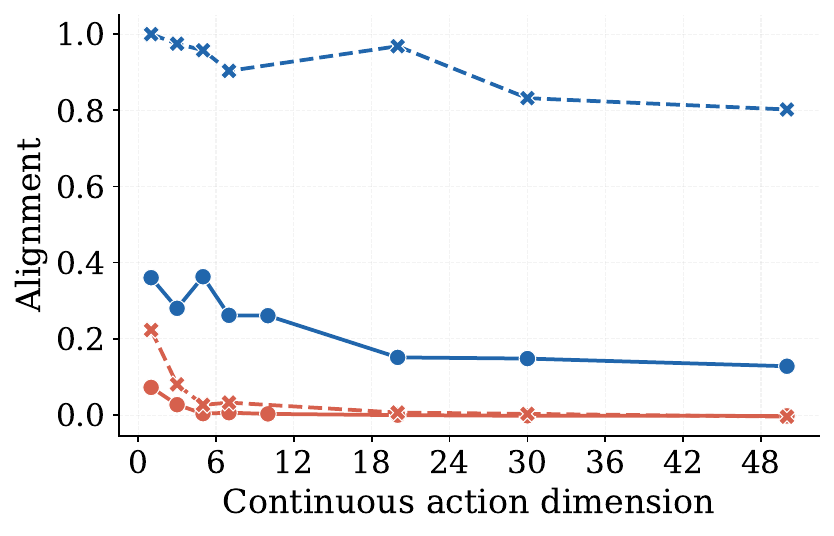}
  \caption{Alignment \textbf{(higher is better)}.}
  \label{fig:alignment-ppo-hybrid}
\end{subfigure}
\hfill
\begin{subfigure}{0.46\textwidth}
  \centering
  \includegraphics[width=\linewidth]{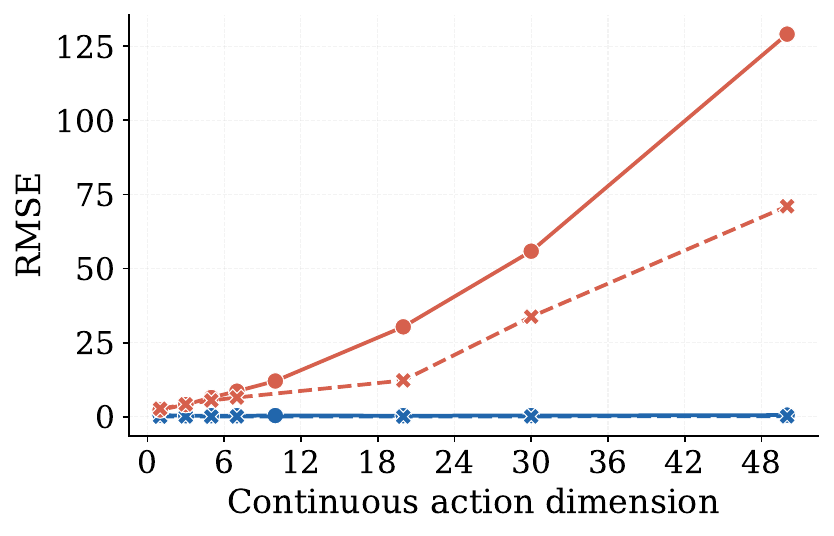}
  \caption{RMSE \textbf{(lower is better)}.}
  \label{fig:rmse-ppo-hybrid}
\end{subfigure}
\caption{
Alignment and RMSE of batch-level gradient estimates for the mixed and SF estimators
for varying continuous action dimensions. Blue and red lines correspond to the mixed
and SF estimators. Solid and dashed lines denote performance gap ranges of $5$--$10\%$
and $30$--$35\%$, respectively.
}
\label{fig:ppo_vs_hybrid}
\end{figure}

\section{The cross term vanishes for near-best response discrete policy components}
\label{sec:cross-term-vanishes}

In many hybrid MDPs, the two components of the action play distinct roles: the discrete component selects a regime — a controller, an operating mode, or a strategic choice — while the continuous component optimizes within it. In practice, distinct teams are often responsible for building a controller for each component. This raises a natural question: \emph{can one have a
decomposed training procedure in which each component is optimized separately using only behavioral access to the other}, that is, using observed states and actions without privileged access to the other component’s policy?

The obstacle to such decomposition lies in the cross term of the mixed gradient. While the discrete component can be optimized using standard SF estimators based only on observed actions and costs, the cross term involves $\nabla_\kappa \log \pi^{\mathcal X}_\phi$, capturing how perturbations of the continuous action influence discrete decisions along the generated trajectories. Computing this term requires access to the internal structure of the discrete policy, not just its realized actions. 
Fortunately, as we show below, the 
cross term becomes small near a discrete best response, enabling approximate decentralization.

We now provide a second motivation for dropping the cross term. 
The cross term's estimation variance grows as the discrete policy approaches a best response (Section~\ref{sec:gradient-quality-exps}, Appendix~\ref{appendix:gradient-quality-results}). Combined with the vanishing-expectation result we prove below, one is estimating a vanishing signal with growing noise, and can improve gradient quality near optimality by dropping the cross term.

Before we state our result, we give some necessary definitions and assumptions.
For fixed continuous policy parameters $\kappa$, define the (discrete policy) \textit{best-response} $\pi^{\mathcal{X}}(\kappa) \in \arg\min_{\pi^{\mathcal{X}}} J(\pi^{\mathcal{X}}, \kappa)$, and the corresponding policy $\pi^*_\kappa = (\pi^{\mathcal{X}}(\kappa), \pi^{\mathcal{B}}_\kappa)$. We say that $\pi^{\mathcal{X}}_\phi$ is an \emph{$\epsilon$-best response} to $\kappa$ if $\mathbb{E}_{s_0 \sim \rho}[V_{\pi_{(\phi, \kappa)}}(s_0) - V_{\pi^*_\kappa}(s_0)] \leq \epsilon$ for a given initial state distribution $\rho$.

\begin{assumption}[Regularity conditions]
\label{ass:regularity-cross-term}
There exist constants $C, L_\pi, L_s < \infty$ such that, for all $\kappa \in \mathcal{K}$:
\begin{enumerate}[leftmargin=*, noitemsep]
    \item \textbf{Statewise best response.} There exists best-response $\pi^{\mathcal{X}}(\kappa)$ such that for all $s \in \mathcal{S}$, $V_{\pi^*_\kappa}(s) = \min_{\pi^{\mathcal{X}}} V_{(\pi^{\mathcal{X}}, \pi^{\mathcal{B}}_\kappa)}(s)$.
    \item \textbf{Uniform near-best response.} If $\pi^{\mathcal{X}}_\phi$ is 
    any $\epsilon$-best response to $\kappa$, then for any initial state distribution $\rho'$, $\mathbb{E}_{s_0 \sim \rho'}[V_{\pi_{\theta}}(s_0) - V_{\pi^*_\kappa}(s_0)] \leq C \epsilon$.
    \item \textbf{Policy smoothness and state sensitivity.} 
    $\|\nabla_\kappa s_t(\kappa)\| \leq L_s$ for all $t\geq 1$, $\xi_{0:t-1}$ and $x_{0:t-1} \in \mathcal X^{t}$, and $\sup_{\phi, s, x} \|\nabla_s \log \pi^{\mathcal{X}}_\phi(x \mid s)\| \leq L_\pi$.
\end{enumerate}
\end{assumption}

We now state our result.
\begin{theorem}[Cross term vanishes for near-best responses]
\label{thm:cross-term-vanishes-for-near-opt}
Under Assumption~\ref{ass:regularity-cross-term}, for any $\kappa$ and any $\epsilon$-best response $\pi^{\mathcal{X}}_\phi$,
\[
\left\|
\mathbb{E}_{\theta}\!\left[
\sum_{t\geq 0}
\gamma^{t}\,
\nabla_{\kappa}\log \pi^{\mathcal{X}}_{\phi}\!\left(x_t \mid s_t(\kappa)\right)\,
Q_{\pi_{\theta}}(s_t(\kappa),a_t)
\right]
\right\|
\;\leq \epsilon  \times  L_{\pi} L_s  \left(1 + \frac{C \gamma}{1 - \gamma} \right).
\]
\end{theorem}

The proof in Appendix~\ref{sec:proof-of-vanishing-cross-term} shows that the magnitude of the cross term is controlled by the suboptimality of the discrete policy relative to a best response, and by the sensitivity of the score term to the state and of the state to $\kappa$. 
These boundedness assumptions are reasonable under smooth stochastic policies (\eg with entropy regularization), which prevent sharp decision boundaries. Analyzing the behavior near deterministic limits is more subtle and we leave it to future work.

Our result may be viewed as an envelope-style property for hybrid control problems in dynamic settings \citep{milgrom2002envelope,danskin1966theory}, suggesting that the coupling between discrete and continuous decisions weakens as the discrete policy approaches a best response. 

To illustrate Theorem~\ref{thm:cross-term-vanishes-for-near-opt} empirically, Figure~\ref{fig:cross-term-signal} reports the gradient norms of the PW and cross terms in the setting of Section~\ref{sec:gradient-quality-exps} with $p=20$. As the discrete policy approaches a best response, the cross term norm decreases to zero even though the overall policy is not yet optimal, as indicated by the nonzero PW norm. This suggests that the continuous component can be updated using only the PW gradient, incurring only a small bias, and supports the view that near a discrete best response the problem admits an approximately decentralized optimization.

To assess the impact of removing the cross term, Figures~\ref{fig:cross-term-jrp}
and~\ref{fig:cross-term-lqr} report the relative change in performance when using
\textbf{HPONoCross} in place of \textbf{HPOFull} on the JRP and S-LQR experiments of
Section~\ref{sec:performance-results}, respectively. In the JRP setting, removing the
cross term typically yields a slight performance degradation of at most $2\%$; the one
exception is the combination $p=60$, $B=64$, where the degradation reaches $7\%$.
These drops suggest that the discrete policy is not yet near a best response in this
setting, so the cross term carries meaningful gradient signal. In the S-LQR problem, by contrast, removing the cross
term is consistently beneficial, with gains that grow steadily with $p$. We analyze
these effects in Appendix~\ref{appendix:gradient-quality-results}, where we show that
dropping the cross term may introduce bias but tends to yield more stable updates in
practice.

Combined with the empirical observation in Appendix~\ref{appendix:gradient-quality-results} that the cross term's \emph{variance} grows as the discrete policy approaches a best response (as previewed at the start of this section), this implies that near optimality we are estimating a vanishing quantity with high-variance noise --- providing a clear rationale for \textbf{HPONoCross} to drop the cross term even in fully centralized settings.

\begin{figure}
\centering
\begin{subfigure}{0.32\textwidth}
    \centering
    \includegraphics[width=\linewidth]{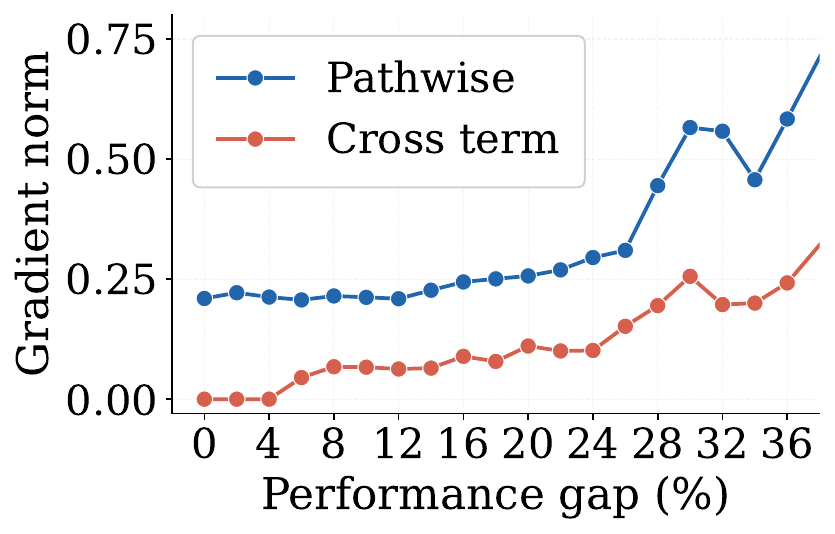}
    \caption{Norm vs. performance gap.}
    \label{fig:cross-term-signal}
\end{subfigure}
\hfill
\begin{subfigure}{0.32\textwidth}
    \centering
    \includegraphics[width=\linewidth]{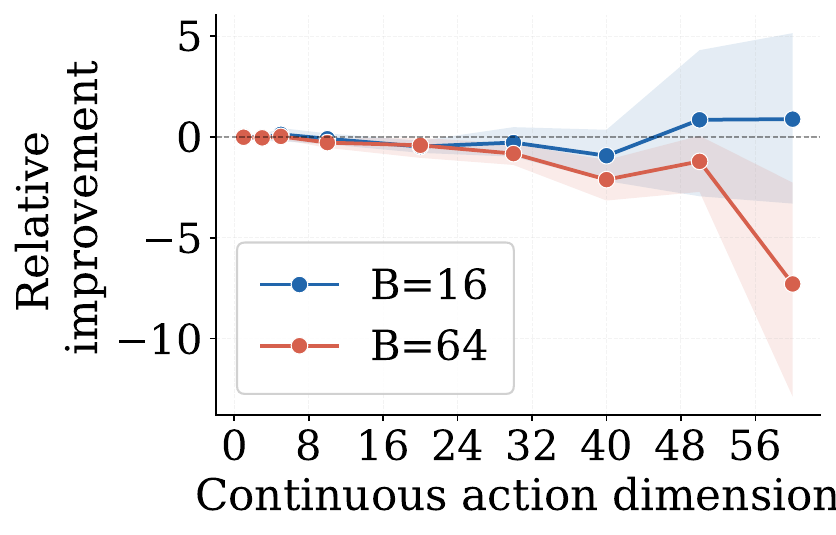}
    \caption{Relative improvement (JRP).}
    \label{fig:cross-term-jrp}
\end{subfigure}
\hfill
\begin{subfigure}{0.32\textwidth}
    \centering
    \includegraphics[width=\linewidth]{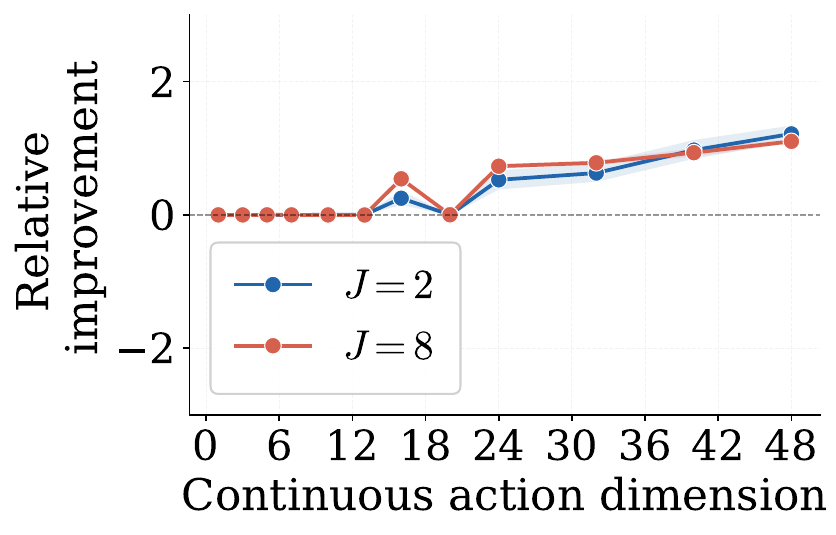}
    \caption{Relative improvement (S-LQR).}
    \label{fig:cross-term-lqr}
\end{subfigure}
\caption{(a) Gradient norms of the PW and cross terms for varying performance gaps in the JRP setting. 
(b) Cross term ablation in the JRP across batch sizes $B$. 
(c) Cross term ablation in the S-LQR across mode counts $J$. 
Improvement is measured as the percentage reduction for \textbf{HPONoCross} relative to \textbf{HPOFull}; positive values indicate better performance with \textbf{HPONoCross}.}
\label{fig:cross-term-ablation}
\end{figure}

\section{Conclusion}
Model-free policy-gradient methods such as REINFORCE and PPO struggle in large action dimensions, driven by the poor quality of the SF-type gradients they rely on. In this work, we propose HPO, which leverages mixed gradients for hybrid discrete-continuous action MDPs. Empirically, HPO overcomes the shortcomings of SF-based gradients in high-dimensional continuous action spaces, representing a significant step towards the applicability of policy-gradient methods in problems of realistic size and complexity. We additionally establish that the cross term of the mixed gradient vanishes as the discrete policy approaches a best response,  
a property which may have implications both for the design of semi-decentralized RL algorithms and for improving gradient quality.

\newpage

\bibliographystyle{abbrvnat}
\bibliography{references}

@book{bertsekas2012dynamic,
  title={Dynamic programming and optimal control: Volume I},
  author={Bertsekas, Dimitri},
  volume={1},
  year={2012},
  publisher={Athena scientific}
}

@article{scarf1960optimality,
  title={The optimality of (S, s) policies in the dynamic inventory problem},
  author={Scarf, Herbert and Arrow, K and Karlin, S and Suppes, P},
  journal={Optimal pricing, inflation, and the cost of price adjustment},
  pages={49--56},
  year={1960},
  publisher={MIT Press Cambridge}
}

@article{schulman2017proximal,
  title={Proximal policy optimization algorithms},
  author={Schulman, John and Wolski, Filip and Dhariwal, Prafulla and Radford, Alec and Klimov, Oleg},
  journal={arXiv preprint arXiv:1707.06347},
  year={2017}
}

@article{levine2016end,
  title={End-to-end training of deep visuomotor policies},
  author={Levine, Sergey and Finn, Chelsea and Darrell, Trevor and Abbeel, Pieter},
  journal={The Journal of Machine Learning Research},
  volume={17},
  number={1},
  pages={1334--1373},
  year={2016},
  publisher={JMLR. org}
}

@article{madeka2022deep,
  title={Deep Inventory Management},
  author={Madeka, Dhruv and Torkkola, Kari and Eisenach, Carson and Foster, Dean and Luo, Anna},
  journal={arXiv preprint arXiv:2210.03137},
  year={2022}
}

@article{williams1992simple,
  title={Simple statistical gradient-following algorithms for connectionist reinforcement learning},
  author={Williams, Ronald J},
  journal={Reinforcement learning},
  pages={5--32},
  year={1992},
  publisher={Springer}
}

@article{sinclair2022hindsight,
  title={Hindsight Learning for MDPs with Exogenous Inputs},
  author={Sinclair, Sean R and Frujeri, Felipe and Cheng, Ching-An and Swaminathan, Adith},
  journal={arXiv preprint arXiv:2207.06272},
  year={2022}
}

@article{hu2019difftaichi,
  title={Difftaichi: Differentiable programming for physical simulation},
  author={Hu, Yuanming and Anderson, Luke and Li, Tzu-Mao and Sun, Qi and Carr, Nathan and Ragan-Kelley, Jonathan and Durand, Fr{\'e}do},
  journal={arXiv preprint arXiv:1910.00935},
  year={2019}
}

@article{mohamed2020monte,
  title={Monte carlo gradient estimation in machine learning},
  author={Mohamed, Shakir and Rosca, Mihaela and Figurnov, Michael and Mnih, Andriy},
  journal={The Journal of Machine Learning Research},
  volume={21},
  number={1},
  pages={5183--5244},
  year={2020},
  publisher={JMLRORG}
}

@article{vanvuchelen2023use,
  title={The Use of Continuous Action Representations to Scale Deep Reinforcement Learning for Inventory Control},
  author={Vanvuchelen, Nathalie and De Moor, Bram and Boute, Robert},
  year={2023}
}

@article{sutton1999policy,
  title={Policy gradient methods for reinforcement learning with function approximation},
  author={Sutton, Richard S and McAllester, David and Singh, Satinder and Mansour, Yishay},
  journal={Advances in neural information processing systems},
  volume={12},
  year={1999}
}

@article{schulman2015gradient,
  title={Gradient estimation using stochastic computation graphs},
  author={Schulman, John and Heess, Nicolas and Weber, Theophane and Abbeel, Pieter},
  journal={Advances in neural information processing systems},
  volume={28},
  year={2015}
}

@misc{kingma2013auto,
  title={Auto-encoding variational bayes},
  author={Kingma, Diederik P and Welling, Max and others},
  year={2013},
  publisher={Banff, Canada}
}

@inproceedings{rezende2014stochastic,
  title={Stochastic backpropagation and approximate inference in deep generative models},
  author={Rezende, Danilo Jimenez and Mohamed, Shakir and Wierstra, Daan},
  booktitle={International conference on machine learning},
  pages={1278--1286},
  year={2014},
  organization={PMLR}
}

@article{bengio2013estimating,
  title={Estimating or propagating gradients through stochastic neurons for conditional computation},
  author={Bengio, Yoshua and L{\'e}onard, Nicholas and Courville, Aaron},
  journal={arXiv preprint arXiv:1308.3432},
  year={2013}
}

@article{xie2024vc,
  title={Vc theory for inventory policies},
  author={Xie, Yaqi and Ma, Will and Xin, Linwei},
  journal={arXiv preprint arXiv:2404.11509},
  year={2024}
}

@article{che2024differentiable,
  title={Differentiable Discrete Event Simulation for Queuing Network Control},
  author={Che, Ethan and Dong, Jing and Namkoong, Hongseok},
  journal={arXiv preprint arXiv:2409.03740},
  year={2024}
}

@inproceedings{levy2018deterministic,
  title={Deterministic policy optimization by combining pathwise and score function estimators for discrete action spaces},
  author={Levy, Daniel and Ermon, Stefano},
  booktitle={Proceedings of the AAAI Conference on Artificial Intelligence},
  volume={32},
  number={1},
  year={2018}
}

@article{xiong2018parametrized,
  title={Parametrized deep q-networks learning: Reinforcement learning with discrete-continuous hybrid action space},
  author={Xiong, Jiechao and Wang, Qing and Yang, Zhuoran and Sun, Peng and Han, Lei and Zheng, Yang and Fu, Haobo and Zhang, Tong and Liu, Ji and Liu, Han},
  journal={arXiv preprint arXiv:1810.06394},
  year={2018}
}

@article{li2021hyar,
  title={Hyar: Addressing discrete-continuous action reinforcement learning via hybrid action representation},
  author={Li, Boyan and Tang, Hongyao and Zheng, Yan and Hao, Jianye and Li, Pengyi and Wang, Zhen and Meng, Zhaopeng and Wang, Li},
  journal={arXiv preprint arXiv:2109.05490},
  year={2021}
}

@article{huang2022mixed,
  title={Mixed deep reinforcement learning considering discrete-continuous hybrid action space for smart home energy management},
  author={Huang, Chao and Zhang, Hongcai and Wang, Long and Luo, Xiong and Song, Yonghua},
  journal={Journal of Modern Power Systems and Clean Energy},
  volume={10},
  number={3},
  pages={743--754},
  year={2022},
  publisher={SGEPRI}
}

@article{milgrom2002envelope,
  title={Envelope theorems for arbitrary choice sets},
  author={Milgrom, Paul and Segal, Ilya},
  journal={Econometrica},
  volume={70},
  number={2},
  pages={583--601},
  year={2002},
  publisher={Wiley Online Library}
}

@article{schulman2015high,
  title={High-dimensional continuous control using generalized advantage estimation},
  author={Schulman, John and Moritz, Philipp and Levine, Sergey and Jordan, Michael and Abbeel, Pieter},
  journal={arXiv preprint arXiv:1506.02438},
  year={2015}
}

@article{zhang2009value,
  title={On the value functions of the discrete-time switched LQR problem},
  author={Zhang, Wei and Hu, Jianghai and Abate, Alessandro},
  journal={IEEE Transactions on Automatic Control},
  volume={54},
  number={11},
  pages={2669--2674},
  year={2009},
  publisher={IEEE}
}

@article{zhang2011infinite,
  title={Infinite-horizon switched LQR problems in discrete time: A suboptimal algorithm with performance analysis},
  author={Zhang, Wei and Hu, Jianghai and Abate, Alessandro},
  journal={IEEE Transactions on Automatic Control},
  volume={57},
  number={7},
  pages={1815--1821},
  year={2011},
  publisher={IEEE}
}

@article{goyal1989joint,
  title={Joint replenishment inventory control: deterministic and stochastic models},
  author={Goyal, Suresh K and Satir, Ahmet T},
  journal={European journal of operational research},
  volume={38},
  number={1},
  pages={2--13},
  year={1989},
  publisher={Elsevier}
}

@article{ata2025computational,
  title={A Computational Method for Solving the Stochastic Joint Replenishment Problem in High Dimensions},
  author={Ata, Bar{\i}{\c{s}} and van Eekelen, Wouter and Zhong, Yuan},
  journal={arXiv preprint arXiv:2511.11830},
  year={2025}
}

@article{kalman1960contributions,
  title={Contributions to the theory of optimal control},
  author={Kalman, Rudolf Emil and others},
  journal={Bol. soc. mat. mexicana},
  volume={5},
  number={2},
  pages={102--119},
  year={1960}
}

@article{danskin1966theory,
  title={The theory of max-min, with applications},
  author={Danskin, John M},
  journal={SIAM Journal on Applied Mathematics},
  volume={14},
  number={4},
  pages={641--664},
  year={1966},
  publisher={SIAM}
}

@article{huang2022cleanrl,
  title={CleanRL: High-quality Single-file Implementations of Deep Reinforcement Learning Algorithms},
  author={Huang, Shengyi and Dossa, Rousslan Fernand Julien and Ye, Chang and Braga, Jeff and Chakraborty, Dipam and Mehta, Kinal and Araújo, João G.M.},
  journal={Journal of Machine Learning Research},
  year={2022},
  volume={23},
  number={274},
  pages={1--18},
}

@article{alvo2023deep,
  title={Deep Reinforcement Learning for Inventory Networks: Toward Reliable Policy Optimization},
  author={Alvo, Matias and Russo, Daniel and Kanoria, Yash and Lee, Minuk},
  journal={arXiv preprint arXiv:2306.11246},
  year={2023}
}

\newpage

\appendix

\paragraph{Organization of the appendix.}
We divide the appendix into four sections.
Appendix~\ref{appendix:problem-and-models} formalizes the problem classes studied in the paper and provides detailed model descriptions. Appendix~\ref{appendix:theory} contains the theoretical results and their proofs, along with gradient derivations for alternative policy classes. Appendix~\ref{appendix:policy-and-implementation} describes the policy parameterization and implementation details used in our experiments. Finally, Appendix~\ref{appendix:results} provides full experimental specifications, additional results, and details of the gradient quality analysis.

\section{Problem setup and models}
\label{appendix:problem-and-models}

This appendix formalizes the problem classes considered in the paper and provides detailed model descriptions. Appendix \ref{appendix:hybrid-mdps} gives a formal definition of hybrid MDPs. Appendix~\ref{appendix:ps-mdps} defines piecewise-smooth MDPs and shows how they can be represented as hybrid MDPs. Appendix~\ref{appendix:reformulation-inv-control} illustrates this reformulation through an inventory control example with fixed costs and varying lead times. Appendix~\ref{appendix:jrp-model} presents a complete formulation of the joint replenishment problem, alongside its hybrid MDP representation.

\subsection{Hybrid discrete-continuous action MDPs}
\label{appendix:hybrid-mdps}

Here, we provide a formal definition of hybrid MDPs.
A \textit{hybrid MDP} has state space 
$\mathcal{S} \subseteq \mathbb{R}^{d_{\mathcal S}}$ for some $d_{\mathcal S} \in \mathbb{N}$, 
and action space of the form
$\mathcal{A} = \bigcup_{j \in \mathcal{X}} \{j\} \times \mathcal{B}^j$,
where $\mathcal{X}$ is a finite set of discrete actions, and for each $j \in \mathcal{X}$, 
$\mathcal{B}^j \subseteq \mathbb{R}^{p}$ denotes the set of feasible continuous controls. 
An action at time $t$ is therefore a pair $a_t = (x_t, b_t)$, 
where $x_t \in \mathcal{X}$ and $b_t \in \mathcal{B}^{x_t}$.

To obtain PW gradients for the continuous action, we impose the following assumption.

\begin{assumption}[Smoothness of dynamics and costs]
\label{assumption:smoothness-for-continuous}
For every discrete action $x$ and exogenous realization $\xi$, the transition and cost functions $(s,b)\mapsto f(s,(x,b),\xi)$ and $(s,b)\mapsto c(s,(x,b),\xi)$ are continuous and almost everywhere (a.e.) differentiable with respect to $(s,b)$.
\end{assumption}

Assumption~\ref{assumption:smoothness-for-continuous} imposes smoothness only with respect to the continuous action $b$, but no conditions on the discrete action $x$. In contrast, applying standard PW methods typically requires problem settings with fully continuous action spaces and global smoothness with respect to all action components.

\subsection{Piecewise smooth MDPs}
\label{appendix:ps-mdps}

MDPs with a finite set of discontinuities in the action are a key class of problems we wish to address under our framework.
We formalize them as Piecewise smooth MDP (PS-MDP) as follows:

\begin{definition}[PS-MDP]
\label{def:ps-mdp}
An MDP $\mathcal{M} = (\mathcal{S}, \mathcal{A}, c, f, \mathbb{P}_{\xi})$ is a piecewise-smooth MDP with $J$ regions if 
i) the action space admits a partition $\mathcal{A} = \bigcup_{j=1}^J \mathcal{B}^j$, where each $\mathcal{B}^j \subseteq \mathbb{R}^p$ is non-empty and connected, and 
ii) for each $j \in [J]$, the restricted functions
\[
c|_{\mathcal{B}^j}(s,a,\xi), \quad f|_{\mathcal{B}^j}(s,a,\xi)
\]
are continuous and almost everywhere differentiable with respect to $a$, for all $(s,\xi)$.
\end{definition}

To represent a PS-MDP as a Hybrid MDP and unlock the benefits of HPO, we exploit 
the partition of the action space given by Definition~\ref{def:ps-mdp}, augmenting 
it into a discrete-continuous representation in which the discrete component selects 
a region and the continuous component specifies a control within that region. 
Crucially, while the original PS-MDP may be discontinuous in $a$, the resulting 
hybrid MDP is smooth in $b$ within each region, satisfying 
Assumption~\ref{assumption:smoothness-for-continuous} and thereby enabling PW 
gradient estimation. We formalize this below, and provide a concrete example in the 
context of inventory control in Appendix~\ref{appendix:reformulation-inv-control}.

\begin{proposition}
\label{prop:ps-mdp-factorization}
Every PS-MDP with $J$ regions admits a representation as a hybrid MDP 
$\tilde{\mathcal{M}}$ satisfying Assumption~\ref{assumption:smoothness-for-continuous} 
such that for every policy $\pi$ in the PS-MDP, there exists a policy $\pi$ 
in $\tilde{\mathcal{M}}$ that induces the same sequence of states and costs for every 
realization $\xi$, and vice versa.
\end{proposition}

\begin{proof}
Define $\tilde{\mathcal{M}}$ with action space
\[
\tilde{\mathcal{A}} = \bigcup_{j=1}^J \{j\} \times \mathcal{B}^j,
\]
transition $f(s,(x,b),\xi) = f^{x}(s,b,\xi)$ and cost 
$c(s,(x,b),\xi) = c^{x}(s,b,\xi)$. 
Assumption~\ref{assumption:smoothness-for-continuous} is satisfied by condition ii) 
of the PS-MDP definition.
For any policy $\pi$ in the PS-MDP, define $\pi$ by setting 
$x_t = j$ whenever $\pi_t(s) \in \mathcal{B}^j$, and $b_t = \pi_t(s)$. 
By construction, $b_t \in \mathcal{B}^{x_t}$ and 
$f^{x_t}(s,b_t,\xi) = f(s,\pi_t(s),\xi)$, with the same holding for costs, so the 
induced trajectories are identical.
Conversely, for any policy $\pi$ in $\tilde{\mathcal{M}}$, define $\pi$ in 
the PS-MDP by setting $\pi_t(s) = b_t$, where $(x_t,b_t) = \pi_t(s)$. 
Since $b_t \in \mathcal{B}^{x_t} \subseteq \mathcal{A}$, $\pi$ is feasible, and the 
induced trajectories are identical by the same argument.
\end{proof}

\subsection{PS-MDP Reformulation: inventory control with fixed costs and varying lead times}
\label{appendix:reformulation-inv-control}
We present a short example to show how to represent a PS-MDP with $J$ regions as a hybrid MDP. 
\begin{example}
\label{example:inv-control}    
Consider the setting of a retailer that sells a unique product across $T$ periods. The retailer pays a linear underage cost $u$ for unmet demand, and a linear holding cost $h$ for carrying inventory at the end of each period. The retailer's supplier charges her fixed costs depending on the order size, as she needs to consolidate the orders into pallets: $0$ for orders between $0$ and $5$ (as a "promotion", the first pallet is free), $10$ for orders between $5$ and $10$, and $20$ for orders between $10$ and $20$. Procurement costs per unit are negligible. Similarly, the supplier's lead time varies depending on the order size, given that she might need to buy additional input materials: orders between $0$ and $7$ can be delivered in $2$ weeks, while orders between $7$ and $20$ in $3$ weeks. They agree that all units in an order will be delivered simultaneously.
\end{example}

\textbf{PS-MDP Representation.} We can model this problem as a PS-MDP with $J=4$ regions as follows.
The system state at time $t$ is $s_t = (I_t, Q_t)$,
where $I_t \in \mathbb{R}$ is the inventory on-hand and $Q_t \in \mathbb{R}_+^2$ tracks outstanding orders up to one minus the maximum lead time of $3$ periods.  
The action at time $t$ is the order quantity $a_t \in [0, 20]$.

Inventory on hand evolves as
\[
I_{t+1} = \max\{I_t - \xi_t, 0\} + Q_t[1],
\]
where $Q_t[i]$ is the $i$-th component of $Q_t$.
The outstanding orders update as
\[
Q_{t+1} = 
\begin{cases}
(Q_t[2] + a_t, 0), & \text{if } 0 \leq a_t \leq 7, \\
(Q_t[2], a_t), & \text{if } 7 < a_t \leq 20.
\end{cases}
\]

Finally, per-period costs are computed as
\[
c(s_t, a_t, \xi_t) = u(\xi_t - I_t)^+ + h(I_t - \xi_t)^+ + \text{FixedCost}(a_t),
\]
with
\[
\text{FixedCost}(a_t) =
\begin{cases}
0, & \text{if } 0 \leq a_t \leq 5, \\
10, & \text{if } 5 < a_t \leq 10, \\
20, & \text{if } 10 < a_t \leq 20.
\end{cases}
\]

\textbf{Hybrid MDP Representation.} We define the hybrid MDP representation of the MDP described above with $J=4$ regions.  
The action space is given by
\[
\mathcal{A} = \bigcup_{j=1}^4\{j\} \times \mathcal{B}^j,
\]
where
\[
\mathcal{B}^1 = [0,5],\quad
\mathcal{B}^2 = (5,7],\quad
\mathcal{B}^3 = (7,10],\quad
\mathcal{B}^4 = (10,20].
\]

The transition function is defined as
\[
f(s_t, (x_t, b_t), \xi_t) = f^{x_t}(s_t, b_t, \xi_t),
\]
where each $f^j$ (which returns a pair $(I_{t+1}, Q_{t+1})$) is given by
\[
f^j(s_t, b_t, \xi_t) =
\begin{cases}
\left( \max\{I_t - \xi_t, 0\} + Q_t[1],\; (Q_t[2] + b_t, 0) \right), & \text{for } j=1,2, \\[6pt]
\left( \max\{I_t - \xi_t, 0\} + Q_t[1],\; (Q_t[2], b_t) \right), & \text{for } j=3,4.
\end{cases}
\]

Similarly, the cost function is given by
\[
c(s_t, (x_t, b_t), \xi_t) = c^{x_t}(s_t, b_t, \xi_t),
\]
where
\[
c^j(s_t, b_t, \xi_t)
=
u(\xi_t - I_t)^+ + h(I_t - \xi_t)^+ + K_j,
\]
with $K_1 = 0,\quad K_2 = 10,\quad K_3 = 10,\quad K_4 = 20$.

\subsection{Detailed Joint Replenishment Problem Formulation}
\label{appendix:jrp-model}

We provide a complete formulation of the joint replenishment problem introduced in
Example~\ref{ex:jrp}. We consider $p$ products over a finite horizon of
$T$ periods. Demand for product $k$ in period $t$ is denoted by $\xi_t^k$, and we write
$\xi_t = (\xi_t^1,\ldots,\xi_t^p)$. We assume a lead time $L \geq 2$ to avoid edge cases.

\textbf{State and action spaces.}
The system state at time $t$ is
\[
s_t = (I_t, Q_t),
\]
where $I_t = (I_t^1,\ldots,I_t^p) \in \mathbb{R}^p$ denotes inventory levels under
backlogged demand, and
\[
Q_t = (Q_t^1,\ldots,Q_t^p),
\qquad Q_t^k \in \mathbb{R}_+^{L-1},
\]
tracks outstanding orders for product $k$ that have not yet arrived.

After observing $s_t$, the retailer chooses an order vector
\[
a_t = (a_t^1,\ldots,a_t^p) \in \mathbb{R}_+^p,
\]
where $a_t^k$ is the order quantity for product $k$.

\textbf{Transition dynamics.}
Inventory evolves according to
\[
I_{t+1}^k = I_t^k - \xi_t^k + Q_t^k[1],
\qquad k \in [p].
\]
Outstanding orders update as
\[
Q_{t+1}^k =
(Q_t^k[2],\ldots,Q_t^k[L-1],a_t^k),
\qquad k \in [p].
\]
That is, the pipeline is shifted one position forward and the newly placed order is inserted
in the last position.

\textbf{Cost function.}
The retailer incurs linear underage and holding costs for each product, as well as a shared
fixed ordering cost whenever at least one product is ordered. The per-period cost is
\[
c(s_t,a_t,\xi_t)
=
\sum_{k=1}^p
\left[
u_k(\xi_t^k - I_t^k)^+
+
h_k(I_t^k - \xi_t^k)^+
\right]
+
K \cdot \mathbf{1}\left\{\sum_{k=1}^p a_t^k > 0\right\}.
\]
Here, $u_k$ and $h_k$ denote the underage and holding costs for product $k$, respectively,
and $K$ is the shared fixed cost incurred whenever the retailer places a positive order for
at least one product.

\subsubsection{Hybrid MDP Reformulation for the Joint Replenishment Problem}
\label{appendix:jrp-reformulation}

We show how the joint replenishment problem in Appendix~\ref{appendix:jrp-model}
can be represented as a hybrid MDP. The discontinuity is induced by the fixed
cost $K$, which is incurred whenever at least one product is ordered.

We define a hybrid action space with two discrete modes:
\[
\mathcal{X} = \{1,2\},
\]
where $x_t = 1$ denotes ``do not order'' and $x_t = 2$ denotes ``place an order.''
The corresponding continuous action regions are
\[
\mathcal{B}^1 = \{0\}^p,
\qquad
\mathcal{B}^2 = \mathbb{R}_+^p \setminus \{0\}^p.
\]
In practice, we replace $\mathcal{B}^1 = \{0\}^p$ with a small hypercube
$\mathcal{B}^1 = [0,\epsilon]^p$ for a small $\epsilon > 0$ to avoid
degenerate gradients; this does not affect the conceptual reformulation.

We define
\[
f^1(s_t,b_t,\xi_t)
=
\left(
I_t^k - \xi_t^k + Q_t^k[1],\;
(Q_t^k[2],\ldots,Q_t^k[L-1],0)
\right)_{k \in [p]},
\]
\[
f^2(s_t,b_t,\xi_t)
=
\left(
I_t^k - \xi_t^k + Q_t^k[1],\;
(Q_t^k[2],\ldots,Q_t^k[L-1],b_t^{k})
\right)_{k \in [p]}.
\]
The transition is therefore
\[
f(s_t,(x_t,b_t),\xi_t)
=
f^{x_t}(s_t,b_t,\xi_t),
\]
with $b_t \in \mathcal{B}^{x_t}$.

Similarly, cost branches are given by
\[
c^1(s_t,b_t,\xi_t)
=
\sum_{k=1}^p
\left[
u_k(\xi_t^k - I_t^k)^+
+
h_k(I_t^k - \xi_t^k)^+
\right],
\]
and
\[
c^2(s_t,b_t,\xi_t)
=
\sum_{k=1}^p
\left[
u_k(\xi_t^k - I_t^k)^+
+
h_k(I_t^k - \xi_t^k)^+
\right]
+
K.
\]
The cost is
\[
c(s_t,(x_t,b_t),\xi_t)
=
c^{x_t}(s_t,b_t,\xi_t).
\]

\section{Theoretical underpinnings}
\label{appendix:theory}

This appendix provides the theoretical results underlying the proposed framework. First, Appendix \ref{appendix:notation-and-assumptions} defines notation and assumptions used throughout this appendix. Appendix~\ref{appendix:hybrid-gradient-proof} derives the mixed gradient estimator for hybrid policies. Appendix~\ref{sec:proof-of-vanishing-cross-term} proves that the cross term vanishes near a discrete best response. Finally, Appendix~\ref{appendix:sf-estimator} presents the SF estimator for fully stochastic policies.

\subsection{Additional notation and assumptions}
\label{appendix:notation-and-assumptions}

Throughout this section, we consider a Hybrid MDP as defined in 
Appendix~\ref{appendix:hybrid-mdps} satisfying 
Assumption~\ref{assumption:smoothness-for-continuous}.

\paragraph{Notation.}
We introduce notation to make explicit the dependence of trajectories on the policy parameters, which will be key for gradient computation. For $t_1 \le t_2$, we write $x_{t_1:t_2} = (x_{t_1}, \ldots, x_{t_2})$ (and similarly for $\xi_{t_1:t_2}$) to denote a subsequence. We consider a stochastic process in which $\theta = (\phi, \kappa)$ induces a joint distribution $\mathcal{D}_{\theta}$ over $(\xi_{0:\infty}, x_{0:\infty})$, and write $\mathbb{E}_\theta[\cdot] = \mathbb{E}_{(\xi_{0:\infty}, x_{0:\infty}) \sim \mathcal{D}_{\theta}}[\cdot]$.

We use parentheses to indicate parameter dependence only when that dependence remains \emph{after realization of the process $(\xi_{0:\infty}, x_{0:\infty})$}. In particular, $\kappa$ influences the distribution $\mathcal{D}_{\theta}$, while, for a fixed realization $(\xi_{0:\infty}, x_{0:\infty}) \sim \mathcal{D}_\theta$, the state trajectory is a deterministic function of the same $\kappa$ (now viewed as a variable argument), which we denote by $s_t(\kappa)$. Conditional on this realization, $s_t(\kappa)$ does not depend on $\phi$, since the discrete actions $x_{0:\infty}$ are fixed. This conditioning highlights the deterministic dependence on $\kappa$ along a realized trajectory, while the distribution of $(\xi_{0:\infty}, x_{0:\infty})$ remains parameter-dependent. We write $b_t(\kappa)$ for the continuous action executed at time $t$, defined as $b_t(\kappa) = b_\kappa(s_t(\kappa), x_t)$, with dependence on $x_t$ implicit. Accordingly, we write costs as $c_t(\kappa) = c\bigl(s_t(\kappa),(x_t,b_t(\kappa)),\xi_t\bigr)$ and transitions as $s_{t+1}(\kappa) = f\bigl(s_t(\kappa),(x_t,b_t(\kappa)),\xi_t\bigr)$.

For a fixed initial state $s_0$, the expected cost is
\begin{equation}
J(\theta) = \mathbb{E}_\theta\!\left[\sum_{t=0}^{\infty}\gamma^t 
c_t(\kappa)\,\middle|\, s_0\right],
\end{equation}
and we denote by $Q_{\pi_\theta}(s,a)$ the corresponding action-value 
function.

\paragraph{Regularity conditions.}
The proof of Theorem~\ref{thm:mixed-gradient} requires the following 
conditions on the policy class, which ensure the gradient estimator is 
well-defined and allow interchange of differentiation and expectation.

\begin{assumption}[Policy regularity]
\label{ass:policy-regularity}
For each $\theta=(\phi,\kappa)$:
\begin{enumerate}
\vspace{-0.2cm}
\setlength{\itemsep}{2pt}
\setlength{\parskip}{0pt}
\setlength{\parsep}{0pt}
    \item The discrete policy $s \mapsto \pi^{\mathcal X}_{\phi}(x \mid s)$ is continuous and a.e. differentiable for every $x\in\mathcal X$.
    \item The continuous policy mapping $\kappa \mapsto b_\kappa(s, x)$ is continuous and a.e. differentiable for every $s \in \mathcal{S}, x\in\mathcal X$.
    \item The following gradients exist a.e. and are uniformly bounded over $t$, $\xi_{0:t}$ and $x_{0:t} \in \mathcal X^{t+1}$:
$\nabla_\phi \log \pi^{\mathcal X}_{\phi}(x_t \mid s_t(\kappa))$,
$\nabla_\kappa \log \pi^{\mathcal X}_{\phi}(x_t \mid s_t(\kappa))$, and
$\nabla_\kappa c(s_t(\kappa),(x_t,b_\kappa(s_t(\kappa),x_t)),\xi_t)$.
    \item The cost function $c(s,a,\xi)$ is uniformly bounded.
\end{enumerate}
\vspace{-0.25cm}
\end{assumption}
These conditions ensure the gradient estimator is well-defined and allow interchange of differentiation and expectation.

\subsection{Proof of Theorem \ref{thm:mixed-gradient}}
\label{appendix:hybrid-gradient-proof}

\begin{proof}

Under Assumption~\ref{ass:policy-regularity}, the mapping
$s \mapsto \pi^{\mathcal{X}}_{\phi}(x \mid s)$ is a.e.
differentiable with bounded log-gradients, which ensures that the SF
estimator is well-defined. Therefore, by the standard policy gradient theorem \citep{sutton1999policy}, we have
\begin{align}
\nabla_\phi J(\theta)
=
\mathbb{E}_{\theta}\!\left[
\sum_{t=0}^{\infty}
\gamma^{t}\,
\nabla_{\phi}\log \pi^{\mathcal{X}}_{\phi}\left(x_t \mid s_t(\kappa)\right)\;
Q_{\pi_{\theta}}\bigl(s_t(\kappa), (x_t, b_t(\kappa))\bigr)
\right].
\end{align}

We now turn to deriving the gradient with respect to the continuous parameters~$\kappa$. The derivation follows the standard approach of differentiating expectations
over trajectories, and is closely related to the gradient estimators derived
in the stochastic computation graph framework of \citet{schulman2015gradient}.

For a fixed initial state $s_0$, we first derive the gradient of the
undiscounted cost at time $t$ with respect to $\kappa$:
\begin{equation}
(1)
\;=\;
\nabla_{\kappa}\,
\mathbb{E}_{\theta}\!\big[
c(s_t(\kappa), (x_t, b_t(\kappa)), \xi_t)
\mid s_0
\big].
\label{eq:single-cost-grad}
\end{equation}

Expanding the nested expectations and interchanging differentiation and integration 
(which is justified by Assumptions~\ref{assumption:smoothness-for-continuous} and~\ref{ass:policy-regularity}, 
which ensure that the relevant derivatives exist almost everywhere, the gradients and costs are uniformly bounded, and hence dominated convergence applies), we obtain
\begin{align*}
(1)
&=
\nabla_{\kappa}
\sum_{x_{0:t} \in \mathcal{X}^{t+1}}
\int
p(\xi_{0:t})
\Bigg(
\prod_{u=0}^{t}
\pi^{\mathcal{X}}_{\phi}\bigl(x_u \mid s_u(\kappa)\bigr)
\Bigg)
c\big(s_t(\kappa), (x_t, b_t(\kappa)), \xi_t\big)
\,d\xi_{0:t}
\\[4pt]
&=
\sum_{x_{0:t} \in \mathcal{X}^{t+1}}
\int
p(\xi_{0:t})
\nabla_{\kappa}
\!\left[
\Bigg(\prod_{u=0}^{t} \pi^{\mathcal{X}}_{\phi}\bigl(x_u \mid s_u(\kappa)\bigr)\Bigg)
c\big(s_t(\kappa), (x_t, b_t(\kappa)), \xi_t\big)
\right]
\,d\xi_{0:t},
\end{align*}
where $p(\xi_{0:t})$ does not depend on $(\phi,\kappa)$.

Applying the product rule, the integrand decomposes into two terms.

For the first term,
\begin{align}
(2)
&=
\Bigg[
\nabla_{\kappa}
\!\left(
\prod_{u=0}^{t} \pi^{\mathcal{X}}_{\phi}\bigl(x_u \mid s_u(\kappa)\bigr)
\right)
\Bigg]
\cdot c\big(s_t(\kappa), (x_t, b_t(\kappa)), \xi_t\big)
\nonumber\\[4pt]
&=
\Bigg[
\sum_{v=0}^{t}
\Bigg(
\prod_{u=0}^{t} \pi^{\mathcal{X}}_{\phi}\bigl(x_u \mid s_u(\kappa)\bigr)
\frac{\nabla_{\kappa} \pi^{\mathcal{X}}_{\phi}\bigl(x_v \mid s_v(\kappa)\bigr)}{\pi^{\mathcal{X}}_{\phi}\bigl(x_v \mid s_v(\kappa)\bigr)}
\Bigg)
\Bigg]
\cdot c\big(s_t(\kappa), (x_t, b_t(\kappa)), \xi_t\big)
\nonumber\\[4pt]
&=
\Bigg(\prod_{u=0}^{t} \pi^{\mathcal{X}}_{\phi}\bigl(x_u \mid s_u(\kappa)\bigr)\Bigg)
\Bigg[
\sum_{v=0}^{t}
\nabla_{\kappa}\log \pi^{\mathcal{X}}_{\phi}\bigl(x_v \mid s_v(\kappa)\bigr)
\Bigg]
\cdot c\big(s_t(\kappa), (x_t, b_t(\kappa)), \xi_t\big).
\label{eq:score_term_hybrid}
\end{align}

For the second term,
\begin{align}
(3)
&=
\Bigg(\prod_{u=0}^{t} \pi^{\mathcal{X}}_{\phi}\bigl(x_u \mid s_u(\kappa)\bigr)\Bigg)
\cdot \nabla_{\kappa}\,c\big(s_t(\kappa), (x_t, b_t(\kappa)), \xi_t\big).
\label{eq:pathwise_term_hybrid}
\end{align}

Combining \eqref{eq:score_term_hybrid} and \eqref{eq:pathwise_term_hybrid}, we obtain
\begin{align*}
(1)
&=
\mathbb{E}_{\theta}\!\left[
\sum_{v=0}^{t}
\nabla_{\kappa}\log \pi^{\mathcal{X}}_{\phi}\!\left(x_v \mid s_v(\kappa)\right)\;
c\big(s_t(\kappa), (x_t, b_t(\kappa)), \xi_t\big)
\;+\;
\nabla_{\kappa} c\big(s_t(\kappa), (x_t, b_t(\kappa)), \xi_t\big)
\right].
\end{align*}

Summing over $t$ incorporating the discount factor $\gamma^t$ and using linearity of expectation,
\begin{align*}
\nabla_\kappa J(\phi,\kappa)
&=
\mathbb{E}_{\theta}\!\Bigg[
\sum_{t=0}^{\infty}\sum_{v=0}^{t}
\gamma^{\,t}\,
\nabla_{\kappa}\log \pi^{\mathcal{X}}_{\phi}\!\left(x_v \mid s_v(\kappa)\right)\;
c\big(s_t(\kappa), (x_t, b_t(\kappa)), \xi_t\big)
\\
&\qquad\qquad\qquad\qquad
+\;
\sum_{t=0}^{\infty}\gamma^{\,t}\,
\nabla_{\kappa} c\big(s_t(\kappa), (x_t, b_t(\kappa)), \xi_t\big)
\Bigg].
\end{align*}

Swapping the order of summation,
\begin{align*}
\nabla_\kappa J(\theta)
&=
\mathbb{E}_{\theta}\!\Bigg[
\sum_{v=0}^{\infty}
\gamma^{\,v}\,
\nabla_{\kappa}\log \pi^{\mathcal{X}}_{\phi}\!\left(x_v \mid s_v(\kappa)\right)
\sum_{t=v}^{\infty} \gamma^{t-v}
c\big(s_t(\kappa), (x_t, b_t(\kappa)), \xi_t\big)
\\
&\qquad\qquad\qquad\qquad
+\;
\sum_{t=0}^{\infty}\gamma^{\,t}\,
\nabla_{\kappa} c\big(s_t(\kappa), (x_t, b_t(\kappa)), \xi_t\big)
\Bigg]
\\[6pt]
&=
\mathbb{E}_{\theta}\!\left[
\sum_{t=0}^{\infty}
\gamma^{\,t}\,
\nabla_{\kappa}\log \pi^{\mathcal{X}}_{\phi}\!\left(x_t \mid s_t(\kappa)\right)\,
Q_{\pi_{\theta}}(s_t(\kappa),a_t)
+
\sum_{t=0}^{\infty}\gamma^{\,t}\,
\nabla_{\kappa} c\big(s_t(\kappa), (x_t, b_t(\kappa)), \xi_t\big)
\right],
\end{align*}
where the final equality follows from the tower property, since the conditional expectation of
$\sum_{t=v}^{\infty}\gamma^{t-v}c(s_t(\kappa),(x_t,b_t(\kappa)),\xi_t)$
given $(s_v(\kappa),a_v)$ is $Q_{\pi_\theta}(s_v(\kappa),a_v)$.

\end{proof}

\subsection{Proof of Theorem~\ref{thm:cross-term-vanishes-for-near-opt}}
\label{sec:proof-of-vanishing-cross-term}

\begin{proof}[Proof of Theorem~\ref{thm:cross-term-vanishes-for-near-opt}]

For a fixed continuous policy parameter $\kappa$, define the post-discrete
action-value function
\[
\tilde Q_{\pi}(s,x)
:=
Q_{\pi}\bigl(s,(x,\pi^{\mathcal B}_{\kappa}(s, x))\bigr),
\]
that is, the value of selecting discrete action $x$ and then applying the
continuous action prescribed by $\pi^{\mathcal B}_{\kappa}$ conditional on $x$. Similarly,
define the post-discrete advantage of the best-response policy $\pi^*_\kappa$ as
\[
\tilde A_{\pi^*_\kappa}(s,x)
:=
\tilde Q_{\pi^*_\kappa}(s,x)-V_{\pi^*_\kappa}(s).
\]
Since $\pi^*_\kappa$ is a statewise best response by
Assumption~\ref{ass:regularity-cross-term}(1), we have
\[
\tilde A_{\pi^*_\kappa}(s,x) \geq 0
\qquad \text{for all } (s,x).
\]

By the performance difference lemma for costs,
\[
\mathbb{E}_{s_0\sim \rho}
\!\left[
V_{\pi_\theta}(s_0)-V_{\pi^*_\kappa}(s_0)
\right]
=
\mathbb{E}_{\theta}\!\left[
\sum_{t\geq 0}\gamma^t
\tilde A_{\pi^*_\kappa}(s_t(\kappa),x_t)
\right].
\]
Since $\pi^{\mathcal X}_{\phi}$ is an $\epsilon$-best response to
$\kappa$, the left-hand side is at most $\epsilon$, and therefore
\begin{equation}
\label{eq:advantage-bound}
\mathbb{E}_{\theta}\!\left[
\sum_{t\geq 0}\gamma^t
\tilde A_{\pi^*_\kappa}(s_t(\kappa),x_t)
\right]
\leq \epsilon.
\end{equation}

Let
\[
\mathcal C
:=
\mathbb{E}_{\theta}\!\left[
\sum_{t\geq 0}
\gamma^t
\nabla_{\kappa}\log \pi^{\mathcal X}_{\phi}
\!\left(x_t\mid s_t(\kappa)\right)
Q_{\pi_\theta}(s_t(\kappa),a_t)
\right]
\]
denote the cross term. Since
\[
Q_{\pi_\theta}(s_t(\kappa),a_t)
=
\tilde Q_{\pi_\theta}(s_t(\kappa),x_t),
\]
we can write
\begin{align*}
\mathcal C
&=
\mathbb{E}_{\theta}\!\left[
\sum_{t\geq 0}\gamma^t
\nabla_{\kappa}\log \pi^{\mathcal X}_{\phi}
\!\left(x_t\mid s_t(\kappa)\right)
\tilde Q_{\pi_\theta}(s_t(\kappa),x_t)
\right] \\
&=
\mathbb{E}_{\theta}\!\left[
\sum_{t\geq 0}\gamma^t
\nabla_{\kappa}\log \pi^{\mathcal X}_{\phi}
\!\left(x_t\mid s_t(\kappa)\right)
\tilde Q_{\pi^*_\kappa}(s_t(\kappa),x_t)
\right] \\
&\quad+
\mathbb{E}_{\theta}\!\left[
\sum_{t\geq 0}\gamma^t
\nabla_{\kappa}\log \pi^{\mathcal X}_{\phi}
\!\left(x_t\mid s_t(\kappa)\right)
\left(
\tilde Q_{\pi_\theta}(s_t(\kappa),x_t)
-
\tilde Q_{\pi^*_\kappa}(s_t(\kappa),x_t)
\right)
\right].
\end{align*}

For fixed realized history up to time $t$ excluding $x_t$, the state
$s_t(\kappa)$ is fixed as a differentiable function of $\kappa$. Hence
\[
\sum_{x\in\mathcal X}
\pi^{\mathcal X}_\phi(x\mid s_t(\kappa))
\nabla_\kappa \log \pi^{\mathcal X}_\phi(x\mid s_t(\kappa))
=
\nabla_\kappa
\sum_{x\in\mathcal X}
\pi^{\mathcal X}_\phi(x\mid s_t(\kappa))
=0.
\]
Thus, any term depending only on $s_t(\kappa)$ may be subtracted from
$\tilde Q_{\pi^*_\kappa}(s_t(\kappa),x_t)$ inside the conditional expectation over $x_t$.
Therefore,
\begin{align*}
\mathcal C
&=
T_1+T_2,
\end{align*}
where
\[
T_1
:=
\mathbb{E}_{\theta}\!\left[
\sum_{t\geq 0}\gamma^t
\nabla_{\kappa}\log \pi^{\mathcal X}_{\phi}
\!\left(x_t\mid s_t(\kappa)\right)
\tilde A_{\pi^*_\kappa}(s_t(\kappa),x_t)
\right],
\]
and
\[
T_2
:=
\mathbb{E}_{\theta}\!\left[
\sum_{t\geq 0}\gamma^t
\nabla_{\kappa}\log \pi^{\mathcal X}_{\phi}
\!\left(x_t\mid s_t(\kappa)\right)
\left(
\tilde Q_{\pi_\theta}(s_t(\kappa),x_t)
-
\tilde Q_{\pi^*_\kappa}(s_t(\kappa),x_t)
\right)
\right].
\]

By Assumption~\ref{ass:regularity-cross-term}(3),
\[
\left\|
\nabla_{\kappa}\log \pi^{\mathcal X}_{\phi}
(x_t\mid s_t(\kappa))
\right\|
\leq
\left\|
\nabla_s \log \pi^{\mathcal X}_{\phi}
(x_t\mid s_t(\kappa))
\right\|
\left\|
\nabla_\kappa s_t(\kappa)
\right\|
\leq
L_\pi L_s.
\]
Using this bound and \eqref{eq:advantage-bound},
\begin{align*}
\|T_1\|
&\leq
\mathbb{E}_{\theta}\!\left[
\sum_{t\geq 0}\gamma^t
\left\|
\nabla_{\kappa}\log \pi^{\mathcal X}_{\phi}
(x_t\mid s_t(\kappa))
\right\|
\tilde A_{\pi^*_\kappa}(s_t(\kappa),x_t)
\right] \\
&\leq
L_\pi L_s\,
\mathbb{E}_{\theta}\!\left[
\sum_{t\geq 0}\gamma^t
\tilde A_{\pi^*_\kappa}(s_t(\kappa),x_t)
\right] \\
&\leq
L_\pi L_s\,\epsilon.
\end{align*}

It remains to bound $T_2$. For any $(s_t,x_t)$, both
$\tilde Q_{\pi_\theta}(s_t(\kappa),x_t)$ and $\tilde Q_{\pi^*_\kappa}(s_t(\kappa),x_t)$ use the
same immediate cost and the same continuous action
$\pi^{\mathcal B}_{\kappa}(s_t(\kappa),x_t)$ in the first period. They differ
only in the continuation policy. Hence,
\begin{align}
\label{eq:q-difference-bound}
\tilde Q_{\pi_\theta}(s_t(\kappa),x_t)
-
\tilde Q_{\pi^*_\kappa}(s_t(\kappa),x_t)
&=
\gamma\,
\mathbb{E}\!\left[
V_{\pi_\theta}(s_{t+1})
-
V_{\pi^*_\kappa}(s_{t+1})
\mid s_t(\kappa),x_t
\right] \\
&\leq
\gamma C\epsilon,
\nonumber
\end{align}
where the inequality follows from Assumption~\ref{ass:regularity-cross-term}(2)
applied to the conditional distribution of $s_{t+1}$ given $(s_t(\kappa),x_t)$.
Moreover, the left-hand side is nonnegative by the statewise optimality of
$\pi^*_\kappa$.

Using \eqref{eq:q-difference-bound} and Assumption~\ref{ass:regularity-cross-term}(3),
\begin{align*}
\|T_2\|
&\leq
\mathbb{E}_{\theta}\!\left[
\sum_{t\geq 0}\gamma^t
\left\|
\nabla_{\kappa}\log \pi^{\mathcal X}_{\phi}
(x_t\mid s_t(\kappa))
\right\|
\left|
\tilde Q_{\pi_\theta}(s_t(\kappa),x_t)
-
\tilde Q_{\pi^*_\kappa}(s_t(\kappa),x_t)
\right|
\right] \\
&\leq
L_\pi L_s
\mathbb{E}_{\theta}\!\left[
\sum_{t\geq 0}\gamma^t
\gamma C\epsilon
\right] \\
&=
L_\pi L_s\,
\frac{C\gamma}{1-\gamma}\,
\epsilon.
\end{align*}

Combining the bounds for $T_1$ and $T_2$ gives
\[
\|\mathcal C\|
\leq
L_\pi L_s\,\epsilon
\left(
1+\frac{C\gamma}{1-\gamma}
\right),
\]
which proves the result.
\end{proof}

\subsection{Score-function estimator for fully stochastic policies}
\label{appendix:sf-estimator}

In this section, we consider a fully stochastic version of the policy, in which
the continuous component is also stochastic. We
represent the policy as first sampling a discrete action
$x_t \sim \pi^{\mathcal X}_{\phi}(\cdot \mid s_t)$, and then sampling
the corresponding continuous action as
$b_t \sim \pi^{\mathcal B}_{\kappa}(\cdot \mid s_t, x_t)$.
Therefore,
\[
\pi_{\theta}(a_t \mid s_t)
=
\pi^{\mathcal X}_{\phi}(x_t \mid s_t)\,
\pi^{\mathcal B}_{\kappa}(b_t \mid s_t, x_t).
\]

Under Assumption~\ref{ass:policy-regularity}, and assuming that the continuous
policy admits a density with respect to Lebesgue measure, whose log-density is
almost everywhere differentiable with bounded gradients, the standard
SF estimator applies \citep{sutton1999policy}. In particular,
\begin{align}
\label{eq:sf-gradient}
\nabla_\theta J(\theta)
=
\mathbb{E}_{\theta}\!\left[
\sum_{t=0}^{\infty}
\gamma^{t}\,
\nabla_{\theta}\log \pi_{\theta}(a_t \mid s_t)\;
Q_{\pi_{\theta}}(s_t,a_t)
\right].
\end{align}

\section{Policy class and implementation}
\label{appendix:policy-and-implementation}

This appendix provides additional details on our policy class and algorithm implementation. Appendix~\ref{appendix:policy-parameterization} details the architecture of the hybrid policy, including its discrete and continuous components and feasibility enforcement mechanisms. Appendix~\ref{appendix:implementation-details} provides implementation details used across experimental settings, including a full algorithmic implementation of HPO (Algorithm~\ref{alg:hpo-ppo}) built on the gradient estimator of Listing~\ref{alg:hpo-loss}.

\subsection{Policy Parameterization}
\label{appendix:policy-parameterization}

For all our experiments, we consider the policy architecture illustrated in Figure~\ref{fig:hybrid-policy-architecture}. The policy is composed of a discrete component and a continuous component, each represented by a neural network.

The discrete policy $\pi^{\mathcal X}_{\phi}$ maps the state $s_t$ to a vector of logits over discrete actions. A softmax transformation is applied to obtain a probability distribution, from which a discrete action $x_t$ is sampled.

The continuous policy $\pi^{\mathcal B}_{\kappa}$ maps the state $s_t$ to a vector of candidate continuous actions $(b_t^1,\dots,b_t^J)$, one for each discrete action. Given the sampled discrete action $x_t$, the executed continuous action is selected as $b_t = b_t^{x_t}$. This parameterization is equivalent in expressive power to defining a conditional continuous policy that takes $(s_t, x_t)$ as input and outputs a single continuous action. We experimented with such a conditional design and observed similar empirical performance.

\paragraph{Differentiable feasibility enforcement.}
We enforce feasibility of the continuous actions through simple output transformations. For the JRP setting (see Example~\ref{ex:jrp}), actions are required to be non-negative, and we therefore apply a softplus transformation to the network outputs. More complex constraints can be enforced as well—for instance, simplex or budget constraints via softmax-type normalizations, or interval constraints via sigmoid-based transformations—but are beyond the scope of this work.

\begin{figure}[t]
    \centering
    \includegraphics[width=0.4\linewidth]{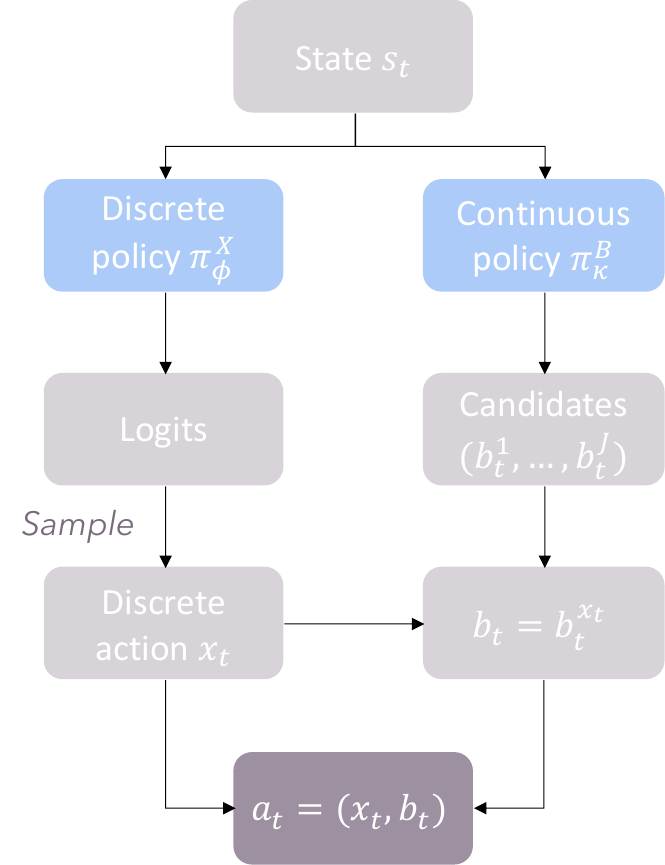}
    \caption{Policy architecture for hybrid MDPs.}
    \label{fig:hybrid-policy-architecture}
\end{figure}

\subsection{Implementation Details}
\label{appendix:implementation-details}

We provide additional details on the implementation of HPO. 
Algorithm~\ref{alg:hpo-ppo} presents a full implementation of HPO built on the loss defined in Listing~\ref{alg:hpo-loss}, reflecting the procedure used in our experiments.
The remainder of this section summarizes the implementation choices and hyperparameters used across all experiments, unless noted otherwise.

\begin{algorithm}[t]
\caption{Hybrid Policy Optimization with PPO-style Updates}
\label{alg:hpo-ppo}
\begin{algorithmic}[1]
\State Initialize policy parameters $\theta=(\phi,\kappa)$ and value parameters $\psi$
\State Let $\mathcal D=\{\xi^h_{0:T-1}\}_{h=1}^{H}$ denote the fixed training dataset of exogenous scenarios

\For{iteration $=1,2,\dots,\texttt{max\_iterations}$}
    \State Partition $\mathcal D$ into training batches
    \For{each training batch $\mathcal{H} \subset \mathcal D$}
        \State Roll out the policy $\pi_\theta=(\pi^{\mathcal X}_\phi,\pi^{\mathcal B}_\kappa)$ on $\mathcal{H}$
        \State Compute costs, value estimates $V_\psi$, and advantages $\widehat A_t$ using GAE

        \Statex
        \State \textbf{Continuous update}
        \State Take one step on $\kappa$ using all trajectories in $\mathcal{H}$:
        \Statex \hspace{\algorithmicindent} -- \textbf{HPOFull}: the full mixed estimator in \eqref{eq:grad-kappa}, with SF terms weighted by $\widehat A_t$;
        \Statex \hspace{\algorithmicindent} -- \textbf{HPONoCross}: only the PW component of \eqref{eq:grad-kappa}.

        \Statex
        \State \textbf{Discrete PPO updates}
        \For{PPO epoch $=1,\dots,\texttt{max\_ppo\_epochs}$}
            \State Partition $\mathcal{H}$ into $\texttt{num\_minibatches}$ minibatches
            \For{each minibatch}
                \State Update $\phi$ using the PPO clipped objective based on \eqref{eq:grad-phi}
                \State Update value parameters $\psi$
            \EndFor
            \State Stop PPO updates early if the approximate KL divergence exceeds the target threshold
        \EndFor

    \EndFor
\EndFor
\end{algorithmic}
\end{algorithm}

\paragraph{Hybrid optimization scheme.}
We optimize the discrete and continuous components of the policy using different gradient estimators. 
The discrete policy parameters are updated using the gradient in \eqref{eq:grad-phi}, implemented with PPO-style clipped objectives as described below. 

For the continuous parameters, \textbf{HPOFull} uses the full mixed gradient estimator in \eqref{eq:grad-kappa}, while \textbf{HPONoCross} uses only the PW component of this estimator. 
For all SF terms (including the cross term), we apply standard PPO modifications, replacing $Q$-values with advantage estimates derived from a learned value function. 

For the PPO baseline, we use the SF gradient in \eqref{eq:sf-gradient} with the same PPO modifications.

\paragraph{Computation of gradient terms.}
We implement a differentiable hybrid simulator in \textsc{PyTorch}, enabling PW gradients through the continuous components via automatic differentiation.
For the PW component, we differentiate the realized trajectory cost with respect to $\kappa$, treating the sampled discrete actions as fixed (i.e., no differentiation through the sampling operation). 
The discrete gradient is computed using SF terms, by differentiating $\log \pi^{\mathcal X}_{\phi}(x_t \mid s_t(\kappa))$ and weighting by advantage estimates. 
For the cross term in \eqref{eq:grad-kappa}, we compute $\nabla_\kappa \log \pi^{\mathcal X}_{\phi}(x_t \mid s_t(\kappa))$ via automatic differentiation through the state trajectory, and multiply it by the corresponding advantage estimate.

\paragraph{Optimization.}
We use Adam with learning rate $10^{-3}$ and $\epsilon = 10^{-5}$. 
Gradients are clipped to have maximum norm $5$, applied independently to each network. 
In practice, for computational efficiency, each batch first receives a full-batch policy update to all parameters, even though Algorithm~\ref{alg:hpo-ppo} presents the discrete and continuous updates separately. We then perform additional PPO epochs in which the data is split into $4$ minibatches and only the discrete policy and value function are updated. For PPO, we follow the CleanRL implementation \cite{huang2022cleanrl}, and thus perform the same number of gradient steps for the continuous and discrete parameters. In Appendix~\ref{appendix:ppo-hyperparams} we consider variants of this scheme to more closely mirror HPO's update structure, but find that they degrade performance. Additionally, for PPO we consider a smaller learning rate of $10^{-4}$
since larger learning rates lead to unstable performance.

\paragraph{Advantage estimation.}
Advantages are computed using generalized advantage estimation (GAE) \cite{schulman2015high} with $\gamma = 0.99$ and $\lambda = 0.96$. 
Advantages are normalized to zero mean and unit variance within each minibatch. 

\paragraph{Policy and value losses.}
We consider a clipping parameter of $0.15$ for the PPO-style clipped objectives. 
The value function is trained using a squared-error loss, and weighted by a coefficient of $0.15$. 
An entropy bonus with initial coefficient $0.5$ is applied to the discrete policy, which is decayed linearly through training.

\paragraph{Network architecture and initialization.}
All networks use two hidden layers of size $512$ with $\tanh$ activations. 
Weights are initialized orthogonally, with gain $\sqrt{2}$ for hidden layers, smaller gains for policy outputs, and gain $1.0$ for the value head. 

\paragraph{Normalization.}
Costs are standardized by dividing by the batch standard deviation, without centering.
For the JRP, state inputs are normalized using a single exponentially weighted estimate of mean demand, computed by averaging demand observations across products and trajectories, with smoothing parameter $\beta=0.99$; continuous-action outputs are rescaled accordingly.

\paragraph{Early stopping.}
Training on a batch is terminated early if the approximate KL divergence exceeds a predefined threshold of 0.015.

\paragraph{Compute resources.}
All experiments were run on NVIDIA A40 GPUs (46 GB memory each), using CUDA 12.2.
Each independent trial was run on a single GPU. Most runs completed in under 60
minutes, with all runs completing in under 90 minutes. These figures are approximate,
as many experiments were run concurrently on a shared cluster without optimizing for
wall-clock time.

\section{Specifications and results for numerical experiments}
\label{appendix:results}

This appendix provides full specifications for the numerical experiments in
Section~\ref{sec:numerical-exps}, along with additional results and hyperparameter details.
Appendix~\ref{appendix:jrp} covers the JRP setting, while Appendix~\ref{appendix:lqr-results}
covers the S-LQR setting, including a validation against Riccati-equation baselines.
Appendix~\ref{appendix:gradient-alignment} provides the full specification of the
gradient-quality experiments of Section \ref{sec:gradient-quality-exps}, including policy construction, gradient computation, metric
definitions, aggregation procedures, and additional results.
Finally, Appendix~\ref{appendix:ppo-hyperparams} reports a hyperparameter robustness 
study for PPO, showing that HPO's advantage persists across a range of PPO configurations.

\subsection{JRP setting (Example~\ref{ex:jrp})}
\label{appendix:jrp}

\textbf{Setting specifications.} We consider Example~\ref{ex:jrp} with heterogeneous 
products. Each product's underage cost, holding cost, and mean demand are sampled 
independently and uniformly from $[6.3, 11.7]$, $[0.7, 1.3]$, and $[6.0, 14.0]$, 
respectively, and are fixed across trajectories. Lead times are set to $L = 2$ periods 
for all products. Demand follows a Poisson distribution with mean given by the sampled 
value. We set the fixed cost to $K = 64p$ to keep the per-product fixed cost constant, 
and vary the number of products $p \in \{1, 3, 5, 10, 20, 30, 40, 50, 60\}$. We consider 
trajectories of $100$ periods and report per-store per-period costs, omitting the first 
and last $20$ periods to approximate the long-run average cost in a stationary system. 
The initial inventory state is set to zero. We consider batch sizes of $B \in \{16, 64\}$ 
trajectories and fix the total number of policy updates to $1600$ by scaling the number 
of full-dataset training iterations as $1600 \cdot B / 1024$. We run $10$ independent 
trials per configuration.

\textbf{Additional specifications.} The policy input is the store inventory vector, 
normalized by a scalar mean demand estimate computed via an exponentially weighted 
moving average with parameter $0.99$, shared across all stores and trajectories. The mean demand estimate and the fixed cost normalized by it are also provided as inputs; while not strictly necessary under i.i.d. demand with a fixed cost, we do so for generality. Network 
outputs are rescaled by the mean demand estimate before being applied to the 
environment. Cost parameters (underage and holding costs) are not included as inputs.

\textbf{Results.} Tables~\ref{tab:jrp-results-16} and~\ref{tab:jrp-results-64} report 
performance metrics for \textbf{HPOFull}, \textbf{HPONoCross}, and PPO across product counts, for batch 
sizes of $16$ and $64$ trajectories, respectively. All metrics consider performance on the test set. All three algorithms perform 
comparably for $p \leq 20$. Beyond this, the effect of batch size becomes pronounced. 
For $B = 16$, PPO degrades sharply from $p = 30$ onward, with mean losses roughly 
doubling relative to HPO by $p = 50$ and standard deviations reaching $4\%$ of the 
mean, reflecting increasing instability. For $B = 64$, PPO remains competitive through 
$p = 40$ but deteriorates at $p = 50$. For $p=60$, PPO's mean loss is roughly $45\%$ 
higher than that of \textbf{HPOFull}. In contrast, both HPO variants maintain stable performance across all product counts and both batch sizes, with standard deviations remaining well below those of PPO throughout.

Comparing \textbf{HPOFull} and \textbf{HPONoCross}, differences are small and inconsistent across most 
settings, with neither variant consistently dominating the other. The most notable 
discrepancy appears at $B = 64$, $p = 60$, where \textbf{HPONoCross}'s mean exceeds that of \textbf{HPOFull} by roughly $7\%$, though both remain far below PPO. Overall, the 
two variants perform similarly, suggesting that the cross term typically provides limited 
additional signal in this setting.

\begin{table}[ht]
  \centering
  \small
  \caption{JRP test loss across product counts for $B=16$ trajectories per batch. Mean $\pm$ std and median over 10 independent trials.}
  \label{tab:jrp-results-16}
  \begin{tabular}{lrrrrrr}
    \toprule
     & \multicolumn{2}{c}{\textbf{HPOFull}} & \multicolumn{2}{c}{\textbf{HPONoCross}} & \multicolumn{2}{c}{\textbf{PPO}} \\
    \cmidrule(lr){2-3} \cmidrule(lr){4-5} \cmidrule(lr){6-7}
    \textbf{$p$} & \textbf{Mean $\pm$ Std} & \textbf{Median} & \textbf{Mean $\pm$ Std} & \textbf{Median} & \textbf{Mean $\pm$ Std} & \textbf{Median} \\
    \midrule
    1 & $34.96 \pm 0.04$ & $34.97$ & $34.97 \pm 0.05$ & $34.95$ & $34.94 \pm 0.03$ & $34.93$ \\
    3 & $37.24 \pm 0.08$ & $37.21$ & $37.27 \pm 0.07$ & $37.29$ & $37.18 \pm 0.06$ & $37.20$ \\
    5 & $38.82 \pm 0.12$ & $38.80$ & $38.77 \pm 0.16$ & $38.76$ & $38.56 \pm 0.05$ & $38.56$ \\
    10 & $38.79 \pm 0.10$ & $38.83$ & $38.85 \pm 0.16$ & $38.86$ & $38.68 \pm 0.07$ & $38.67$ \\
    20 & $39.08 \pm 0.12$ & $39.09$ & $39.26 \pm 0.16$ & $39.24$ & $39.23 \pm 0.08$ & $39.22$ \\
    30 & $39.43 \pm 0.38$ & $39.33$ & $39.54 \pm 0.27$ & $39.53$ & $51.67 \pm 2.87$ & $51.85$ \\
    40 & $40.40 \pm 0.72$ & $40.28$ & $40.82 \pm 0.49$ & $40.98$ & $70.37 \pm 1.92$ & $69.73$ \\
    50 & $43.47 \pm 2.19$ & $42.77$ & $42.95 \pm 1.43$ & $42.44$ & $81.36 \pm 3.27$ & $80.18$ \\
    60 & $50.28 \pm 3.12$ & $49.71$ & $49.92 \pm 1.75$ & $50.39$ & $89.29 \pm 2.76$ & $88.59$ \\
    \bottomrule
  \end{tabular}
\end{table}

\begin{table}[ht]
  \centering
  \small
  \caption{JRP test loss across product counts for $B=64$ trajectories per batch. Mean $\pm$ std and median over 10 independent trials.}
  \label{tab:jrp-results-64}
  \begin{tabular}{lrrrrrr}
    \toprule
     & \multicolumn{2}{c}{\textbf{HPOFull}} & \multicolumn{2}{c}{\textbf{HPONoCross}} & \multicolumn{2}{c}{\textbf{PPO}} \\
    \cmidrule(lr){2-3} \cmidrule(lr){4-5} \cmidrule(lr){6-7}
    \textbf{$p$} & \textbf{Mean $\pm$ Std} & \textbf{Median} & \textbf{Mean $\pm$ Std} & \textbf{Median} & \textbf{Mean $\pm$ Std} & \textbf{Median} \\
    \midrule
    1 & $34.94 \pm 0.02$ & $34.93$ & $34.95 \pm 0.03$ & $34.96$ & $34.93 \pm 0.02$ & $34.94$ \\
    3 & $37.09 \pm 0.04$ & $37.10$ & $37.11 \pm 0.04$ & $37.10$ & $37.09 \pm 0.03$ & $37.09$ \\
    5 & $38.56 \pm 0.08$ & $38.58$ & $38.54 \pm 0.06$ & $38.55$ & $38.45 \pm 0.02$ & $38.45$ \\
    10 & $38.60 \pm 0.04$ & $38.61$ & $38.70 \pm 0.14$ & $38.68$ & $38.52 \pm 0.02$ & $38.52$ \\
    20 & $38.83 \pm 0.04$ & $38.82$ & $38.98 \pm 0.33$ & $38.87$ & $38.85 \pm 0.01$ & $38.86$ \\
    30 & $38.46 \pm 0.06$ & $38.44$ & $38.81 \pm 0.33$ & $38.72$ & $38.51 \pm 0.03$ & $38.50$ \\
    40 & $38.87 \pm 0.29$ & $38.76$ & $39.66 \pm 0.55$ & $39.46$ & $38.80 \pm 0.02$ & $38.80$ \\
    50 & $39.66 \pm 0.71$ & $39.97$ & $40.10 \pm 0.50$ & $40.00$ & $48.75 \pm 2.20$ & $48.89$ \\
    60 & $40.70 \pm 0.98$ & $40.34$ & $43.37 \pm 3.17$ & $41.94$ & $59.18 \pm 1.62$ & $59.09$ \\
    \bottomrule
  \end{tabular}
\end{table}

Then, Figure~\ref{fig:scaling-jrp-targets} reports the median number of policy updates required to reach a target validation performance gap, for batch sizes $B = 16$ and $B = 64$ and target gaps ranging from $5\%$ to $30\%$. For a fixed batch size, HPO consistently reaches the target in fewer updates than PPO, with the gap widening as the continuous action dimension $p$ grows. When comparing across batch sizes, PPO with $B = 64$ is faster than HPO with $B = 16$ at small $p$, as the larger batch compensates for PPO's weaker gradient estimates. However, PPO's convergence speed degrades more steeply with $p$, and for $p\geq 40$ HPO with $B = 16$ converges faster than PPO with $B = 64$. For the stricter targets ($5\%$ and $10\%$), PPO frequently fails to reach the target within the training budget at larger $p$, as indicated by the truncated curves, while HPO converges more reliably, especially for $B=64$.

\begin{figure}
\centering
\begin{subfigure}{.47\textwidth}
  \centering
  \includegraphics[width=1\linewidth]{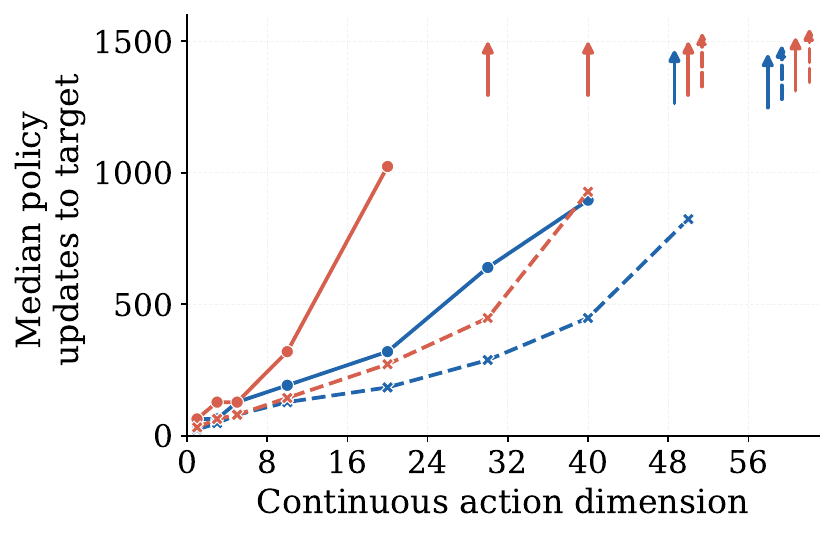}
  \caption{Target gap: $5\%$.}
  \label{fig:scaling-jrp-5}
\end{subfigure}%
\hspace{0.02\textwidth}
\begin{subfigure}{.47\textwidth}
  \centering
  \includegraphics[width=1\linewidth]{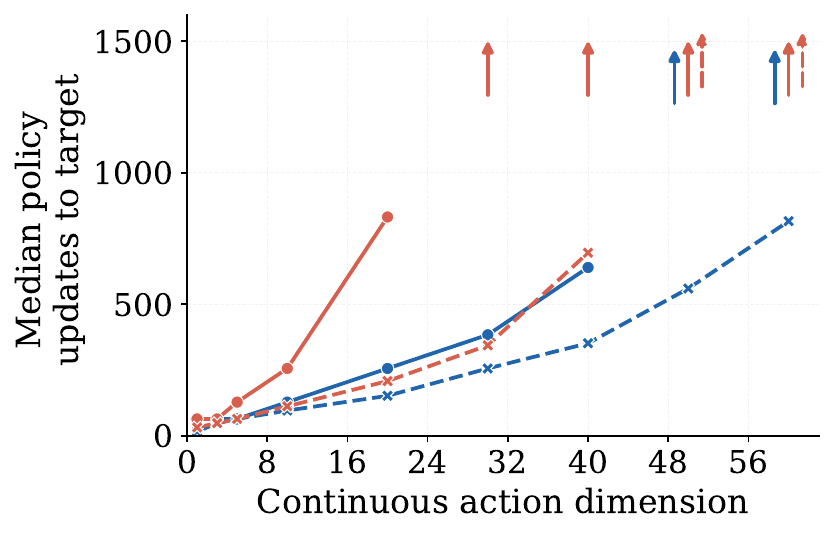}
  \caption{Target gap: $10\%$.}
  \label{fig:scaling-jrp-10}
\end{subfigure}
\vspace{0.5em}
\begin{subfigure}{.47\textwidth}
  \centering
  \includegraphics[width=1\linewidth]{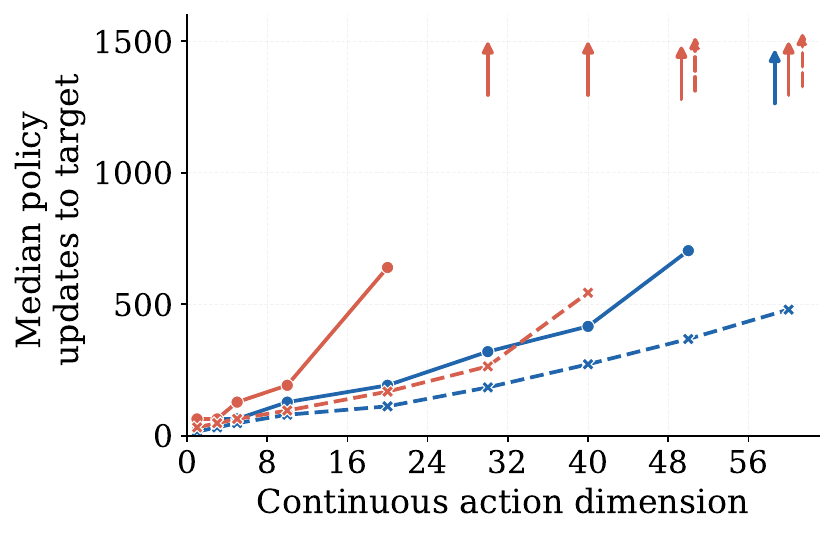}
  \caption{Target gap: $20\%$.}
  \label{fig:scaling-jrp-20}
\end{subfigure}%
\hspace{0.02\textwidth}
\begin{subfigure}{.47\textwidth}
  \centering
  \includegraphics[width=1\linewidth]{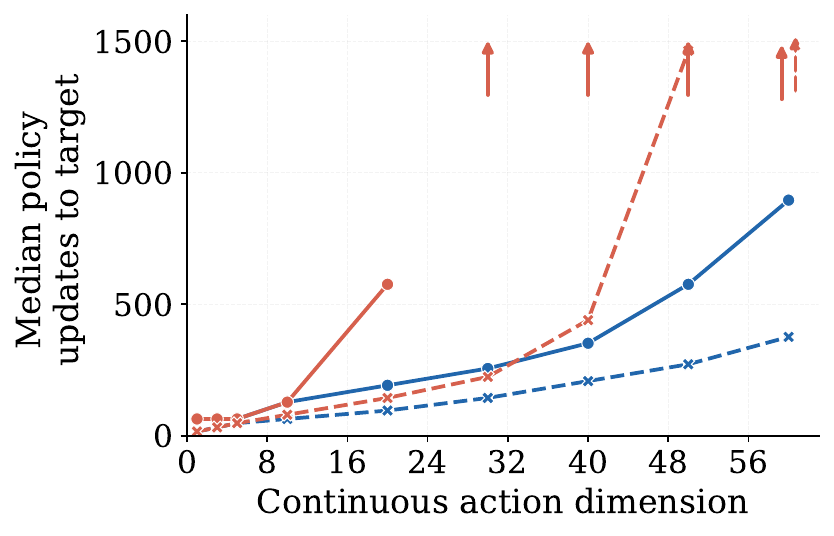}
  \caption{Target gap: $30\%$.}
  \label{fig:scaling-jrp-30}
\end{subfigure}
\caption{Median number of policy updates required to reach a target validation-performance gap in the joint replenishment problem, as the continuous action dimension $p$ varies. Each panel uses a different target gap. Blue curves correspond to HPO and red curves to PPO. Solid lines use batch size $B=16$ trajectories, while dashed lines use batch size $B=64$ trajectories. Arrows indicate that the median run did not converge for the given configuration.}
\label{fig:scaling-jrp-targets}
\end{figure}

\subsection{S-LQR setting (Example \ref{ex:slqr}) \label{appendix:lqr-results}}

\textbf{Setting specifications.} We consider Example~\ref{ex:slqr} with matrices 
constructed so that the system is unstable in the absence of a control. The state 
dimension and continuous action dimension are set to a common value 
$p \in \{1, 3, 5, 7, 10, 13, 16, 20, 24, 32, 40, 48\}$, and the number of modes to 
$J \in \{2, 4, 6, 8\}$. We consider trajectories of $20$ periods, as the control 
typically drives the state to a negligible-cost regime well within this horizon, and 
report total cost over the $20$ periods. We consider $500$ training iterations with a 
batch size of $128$ trajectories, corresponding to $500 \cdot 1024/128 = 4000$ policy 
updates. We run $10$ independent trials per configuration.

For each configuration, we generate a shared dynamics matrix $A \in \mathbb{R}^{p 
\times p}$ and mode-dependent control matrices $\{B_j\}_{j=1}^J$ as follows. The 
dynamics matrix is defined as
\[
A = 1.10\, I + 0.05\, N,
\]
where $I \in \mathbb{R}^{p \times p}$ is the identity matrix and $N \in \mathbb{R}^{p 
\times p}$ is a symmetric random matrix constructed by sampling entries i.i.d.\ from 
$\mathcal{N}(0,1)$, symmetrizing via $N \leftarrow (N + N^\top)/2$, and normalizing so 
that $\max_{i,j} |N_{ij}| = 1$. The scalar $1.10 > 1$ ensures instability in the 
absence of a control, while the perturbation $0.05\,N$ introduces coupling across state 
dimensions without significantly altering the overall dynamics. The state coordinates 
are partitioned into $J$ disjoint groups $\{G_j\}_{j=1}^J$ of approximately equal size. 
For each mode $j$, the control matrix $B_j \in \mathbb{R}^{p \times p}$ is diagonal 
with entries
\[
(B_j)_{kk} =
\begin{cases}
1.0, & \text{if } k \in G_j,\\
0.15, & \text{otherwise,}
\end{cases}
\]
ensuring that each mode provides strong control over a distinct subset of state 
dimensions and weaker control elsewhere, inducing state-dependent optimal mode selection 
while permitting controls that can stabilize the system. The cost matrices are fixed 
across modes and given by $Q_j = I$ and $R_j = 0.1\,I$ for all $j$, corresponding to a 
quadratic penalty on state deviation and a relatively low penalty on control effort. 
This encourages aggressive stabilization and ensures that problem difficulty scales 
primarily with $p$. Each state coordinate is initialized independently and uniformly in 
$[-10/\sqrt{p},\, 10/\sqrt{p}]$ and reset independently across trajectories, keeping 
the expected squared distance from the origin constant across dimensions.

\textbf{Additional specifications.} The raw state vector is passed directly as input to 
all networks without normalization, since the initial state scaling already keeps input 
magnitudes comparable across dimensions.

\textbf{Results.} Table~\ref{tab:lqr-results} reports performance metrics for 
\textbf{HPOFull}, \textbf{HPONoCross}, and PPO across continuous action dimensions $p$ and mode counts 
$J$. Both HPO variants consistently outperform PPO across all configurations, with the 
gap growing steadily as $p$ increases. This pattern holds across all values of $J$, 
confirming that the continuous action dimension is the key driver of performance 
differences. For instance, at $p=48$, $J=2$, PPO achieves 
a mean loss of $79.36$ against \textbf{HPOFull}'s $53.89$, a gap of roughly $47\%$. At small 
dimensions ($p \leq 7$), the gap is modest, typically below $1\%$, but grows 
consistently with $p$. PPO's standard deviations remain low throughout, indicating 
stable but suboptimal convergence. In contrast, both HPO variants maintain low variance 
and lower losses across all $p$ and $J$, highlighting the practical benefits of the 
mixed gradient estimator in high-dimensional continuous action spaces.

Comparing \textbf{HPOFull} and \textbf{HPONoCross}, dropping the cross term provides a consistent but 
modest benefit that grows with $p$. At small dimensions ($p \leq 13$), the two variants 
perform nearly identically across all $J$ values, with differences below $0.01$. From 
$p = 16$ onward, \textbf{HPONoCross} consistently outperforms \textbf{HPOFull}, and the gap widens 
steadily: at $p = 48$, \textbf{HPONoCross} outperforms \textbf{HPOFull} by roughly $0.6$--$1.2$ units 
across all $J$, representing a relative improvement of approximately $0.6\%$--$1.5\%$. 
While these differences are small in absolute terms, their consistency across all $J$ 
values and their systematic growth with $p$ suggest that the cross term becomes 
increasingly detrimental as the continuous action space grows, which we conjecture is 
due to increased variance.

\begin{table}[ht]
  \centering
  \small
  \caption{S-LQR test loss across action dimensions $p$ and modes $J$. Mean $\pm$ std and median over 10 independent trials.}
  \label{tab:lqr-results}
  \begin{tabular}{llrrrrrr}
    \toprule
     &  & \multicolumn{2}{c}{\textbf{HPOFull}} & \multicolumn{2}{c}{\textbf{HPONoCross}} & \multicolumn{2}{c}{\textbf{PPO}} \\
    \cmidrule(lr){3-4} \cmidrule(lr){5-6} \cmidrule(lr){7-8}
    \textbf{$p$} & \textbf{$J$} & \textbf{Mean $\pm$ Std} & \textbf{Median} & \textbf{Mean $\pm$ Std} & \textbf{Median} & \textbf{Mean $\pm$ Std} & \textbf{Median} \\
    \midrule
    1 & $2$ & $36.82 \pm 0.00$ & $36.82$ & $36.82 \pm 0.00$ & $36.82$ & $36.84 \pm 0.00$ & $36.84$ \\
    1 & $4$ & $36.82 \pm 0.00$ & $36.82$ & $36.82 \pm 0.00$ & $36.82$ & $36.84 \pm 0.00$ & $36.84$ \\
    1 & $6$ & $37.50 \pm 0.00$ & $37.50$ & $37.50 \pm 0.00$ & $37.50$ & $37.52 \pm 0.00$ & $37.52$ \\
    1 & $8$ & $37.50 \pm 0.00$ & $37.50$ & $37.50 \pm 0.00$ & $37.50$ & $37.52 \pm 0.00$ & $37.52$ \\
    3 & $2$ & $64.17 \pm 0.00$ & $64.17$ & $64.17 \pm 0.00$ & $64.17$ & $64.33 \pm 0.04$ & $64.31$ \\
    3 & $4$ & $84.83 \pm 0.00$ & $84.83$ & $84.83 \pm 0.00$ & $84.83$ & $85.08 \pm 0.02$ & $85.08$ \\
    3 & $6$ & $79.42 \pm 0.00$ & $79.42$ & $79.42 \pm 0.00$ & $79.42$ & $79.66 \pm 0.05$ & $79.65$ \\
    3 & $8$ & $79.42 \pm 0.00$ & $79.42$ & $79.42 \pm 0.00$ & $79.42$ & $79.65 \pm 0.04$ & $79.64$ \\
    5 & $2$ & $63.92 \pm 0.00$ & $63.92$ & $63.92 \pm 0.00$ & $63.92$ & $64.25 \pm 0.04$ & $64.26$ \\
    5 & $4$ & $85.01 \pm 0.00$ & $85.01$ & $85.01 \pm 0.00$ & $85.01$ & $85.56 \pm 0.06$ & $85.56$ \\
    5 & $6$ & $94.36 \pm 0.00$ & $94.36$ & $94.36 \pm 0.00$ & $94.36$ & $95.00 \pm 0.03$ & $95.01$ \\
    5 & $8$ & $94.36 \pm 0.00$ & $94.36$ & $94.36 \pm 0.00$ & $94.36$ & $94.99 \pm 0.06$ & $95.01$ \\
    7 & $2$ & $67.25 \pm 0.00$ & $67.25$ & $67.25 \pm 0.00$ & $67.25$ & $67.76 \pm 0.04$ & $67.77$ \\
    7 & $4$ & $89.03 \pm 0.00$ & $89.03$ & $89.03 \pm 0.00$ & $89.03$ & $89.82 \pm 0.04$ & $89.84$ \\
    7 & $6$ & $89.24 \pm 0.00$ & $89.24$ & $89.24 \pm 0.00$ & $89.24$ & $90.10 \pm 0.04$ & $90.10$ \\
    7 & $8$ & $99.02 \pm 0.00$ & $99.02$ & $99.02 \pm 0.00$ & $99.02$ & $99.97 \pm 0.05$ & $99.97$ \\
    8 & $2$ & $49.05 \pm 0.06$ & $49.07$ & $48.95 \pm 0.06$ & $48.93$ & $70.66 \pm 0.02$ & $70.67$ \\
    8 & $4$ & $79.51 \pm 0.04$ & $79.51$ & $79.28 \pm 0.03$ & $79.28$ & $93.33 \pm 0.05$ & $93.33$ \\
    8 & $6$ & $79.20 \pm 0.04$ & $79.21$ & $78.93 \pm 0.02$ & $78.93$ & $93.17 \pm 0.03$ & $93.18$ \\
    8 & $8$ & $95.90 \pm 0.03$ & $95.90$ & $95.51 \pm 0.01$ & $95.51$ & $103.81 \pm 0.10$ & $103.80$ \\
    10 & $2$ & $76.39 \pm 0.00$ & $76.39$ & $76.39 \pm 0.00$ & $76.39$ & $77.26 \pm 0.02$ & $77.26$ \\
    10 & $4$ & $86.52 \pm 0.00$ & $86.52$ & $86.52 \pm 0.00$ & $86.52$ & $87.66 \pm 0.04$ & $87.68$ \\
    10 & $6$ & $95.94 \pm 0.00$ & $95.94$ & $95.94 \pm 0.00$ & $95.94$ & $97.34 \pm 0.11$ & $97.37$ \\
    10 & $8$ & $95.94 \pm 0.00$ & $95.94$ & $95.94 \pm 0.00$ & $95.94$ & $97.34 \pm 0.06$ & $97.34$ \\
    13 & $2$ & $71.12 \pm 0.00$ & $71.12$ & $71.12 \pm 0.00$ & $71.12$ & $72.29 \pm 0.12$ & $72.28$ \\
    13 & $4$ & $87.34 \pm 0.00$ & $87.34$ & $87.34 \pm 0.00$ & $87.34$ & $88.80 \pm 0.05$ & $88.79$ \\
    13 & $6$ & $94.32 \pm 0.00$ & $94.32$ & $94.31 \pm 0.00$ & $94.31$ & $95.99 \pm 0.07$ & $96.01$ \\
    13 & $8$ & $100.00 \pm 0.00$ & $100.00$ & $100.00 \pm 0.00$ & $100.00$ & $101.93 \pm 0.08$ & $101.95$ \\
    16 & $2$ & $52.14 \pm 0.06$ & $52.12$ & $52.01 \pm 0.08$ & $52.03$ & $76.36 \pm 0.06$ & $76.33$ \\
    16 & $4$ & $80.18 \pm 0.05$ & $80.17$ & $79.81 \pm 0.05$ & $79.80$ & $93.28 \pm 0.03$ & $93.28$ \\
    16 & $6$ & $90.41 \pm 0.03$ & $90.42$ & $89.90 \pm 0.02$ & $89.90$ & $101.54 \pm 0.05$ & $101.53$ \\
    16 & $8$ & $93.63 \pm 0.03$ & $93.63$ & $93.12 \pm 0.02$ & $93.12$ & $102.38 \pm 0.04$ & $102.40$ \\
    20 & $2$ & $73.77 \pm 0.01$ & $73.77$ & $73.77 \pm 0.01$ & $73.77$ & $75.75 \pm 0.29$ & $75.70$ \\
    20 & $4$ & $92.98 \pm 0.00$ & $92.98$ & $92.98 \pm 0.00$ & $92.98$ & $95.56 \pm 0.07$ & $95.56$ \\
    20 & $6$ & $94.85 \pm 0.00$ & $94.85$ & $94.85 \pm 0.00$ & $94.85$ & $97.44 \pm 0.17$ & $97.40$ \\
    20 & $8$ & $98.68 \pm 0.00$ & $98.68$ & $98.67 \pm 0.00$ & $98.67$ & $101.46 \pm 0.17$ & $101.42$ \\
    24 & $2$ & $52.65 \pm 0.11$ & $52.68$ & $52.37 \pm 0.08$ & $52.38$ & $76.00 \pm 0.06$ & $76.03$ \\
    24 & $4$ & $80.82 \pm 0.10$ & $80.79$ & $80.27 \pm 0.03$ & $80.26$ & $94.70 \pm 0.09$ & $94.69$ \\
    24 & $6$ & $90.97 \pm 0.07$ & $90.96$ & $90.35 \pm 0.02$ & $90.35$ & $101.80 \pm 0.23$ & $101.70$ \\
    24 & $8$ & $97.14 \pm 0.06$ & $97.13$ & $96.43 \pm 0.02$ & $96.43$ & $106.05 \pm 0.15$ & $106.01$ \\
    32 & $2$ & $53.19 \pm 0.09$ & $53.18$ & $52.86 \pm 0.08$ & $52.85$ & $81.43 \pm 6.64$ & $78.01$ \\
    32 & $4$ & $81.45 \pm 0.10$ & $81.45$ & $80.80 \pm 0.07$ & $80.81$ & $94.89 \pm 0.38$ & $94.83$ \\
    32 & $6$ & $89.24 \pm 0.08$ & $89.24$ & $88.52 \pm 0.03$ & $88.52$ & $100.99 \pm 0.19$ & $101.04$ \\
    32 & $8$ & $96.97 \pm 0.05$ & $96.97$ & $96.22 \pm 0.03$ & $96.21$ & $107.25 \pm 0.46$ & $107.06$ \\
    40 & $2$ & $53.33 \pm 0.12$ & $53.30$ & $52.82 \pm 0.06$ & $52.81$ & $77.22 \pm 0.30$ & $77.08$ \\
    40 & $4$ & $82.26 \pm 0.10$ & $82.30$ & $81.31 \pm 0.05$ & $81.30$ & $97.43 \pm 0.11$ & $97.44$ \\
    40 & $6$ & $92.09 \pm 0.09$ & $92.06$ & $91.17 \pm 0.04$ & $91.17$ & $106.11 \pm 0.59$ & $105.86$ \\
    40 & $8$ & $96.37 \pm 0.10$ & $96.35$ & $95.47 \pm 0.04$ & $95.46$ & $107.99 \pm 0.43$ & $108.14$ \\
    48 & $2$ & $53.89 \pm 0.11$ & $53.89$ & $53.24 \pm 0.04$ & $53.24$ & $79.36 \pm 0.27$ & $79.29$ \\
    48 & $4$ & $84.09 \pm 0.11$ & $84.08$ & $82.87 \pm 0.04$ & $82.89$ & $103.12 \pm 3.82$ & $101.53$ \\
    48 & $6$ & $92.98 \pm 0.07$ & $92.99$ & $91.85 \pm 0.04$ & $91.84$ & $107.19 \pm 0.47$ & $106.97$ \\
    48 & $8$ & $97.49 \pm 0.04$ & $97.50$ & $96.41 \pm 0.03$ & $96.42$ & $109.86 \pm 0.73$ & $109.51$ \\
    \bottomrule
  \end{tabular}
\end{table}

Then, Figures~\ref{fig:scaling-lqr-targets} and~\ref{fig:scaling-lqr-v2-targets} report 
the median number of policy updates to reach target gaps ranging from $5\%$ to $30\%$, 
for $J \in \{2, 8\}$ and $J \in \{4, 6\}$, respectively. The performance gap between 
HPO and PPO is drastic and consistent across all values of $J$: PPO requires 
substantially more updates than HPO at moderate $p$, and frequently fails to converge 
within the training budget at larger $p$. HPO, by contrast, converges reliably across 
all configurations, with the required number of updates increasing slowly with $p$.

\begin{figure}
\centering
\begin{subfigure}{.47\textwidth}
  \centering
  \includegraphics[width=1\linewidth]{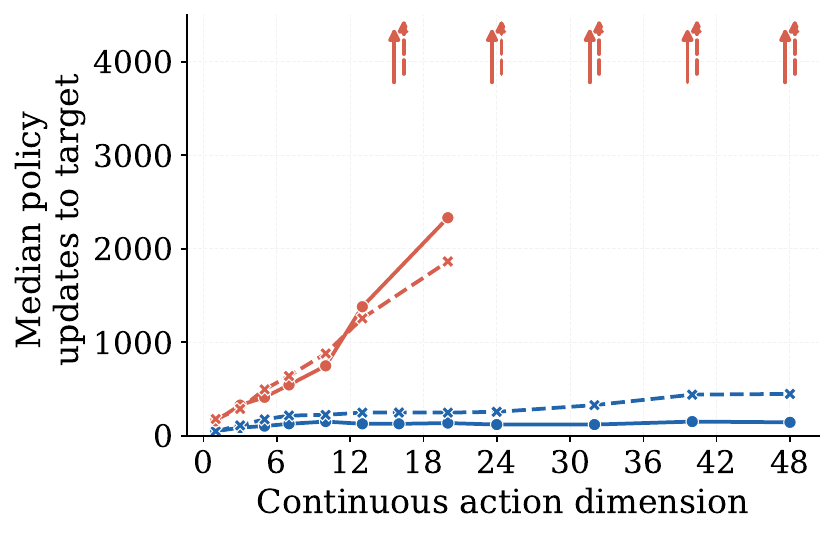}
  \caption{Target gap: $5\%$.}
  \label{fig:scaling-lqr-5}
\end{subfigure}%
\hspace{0.02\textwidth}
\begin{subfigure}{.47\textwidth}
  \centering
  \includegraphics[width=1\linewidth]{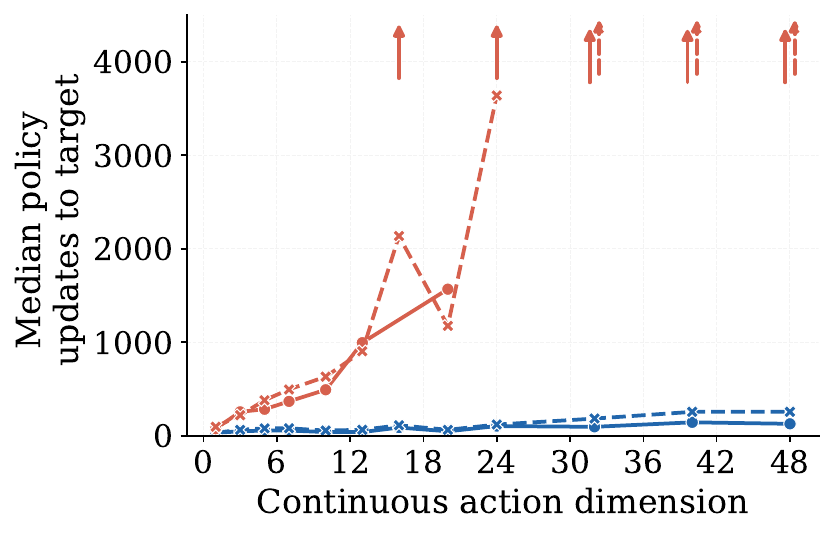}
  \caption{Target gap: $10\%$.}
  \label{fig:scaling-lqr-10}
\end{subfigure}
\vspace{0.5em}
\begin{subfigure}{.47\textwidth}
  \centering
  \includegraphics[width=1\linewidth]{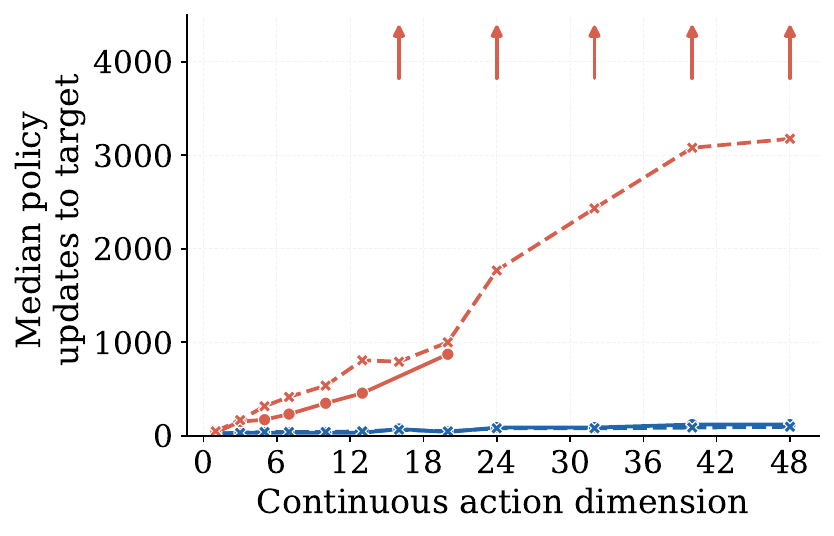}
  \caption{Target gap: $20\%$.}
  \label{fig:scaling-lqr-20}
\end{subfigure}%
\hspace{0.02\textwidth}
\begin{subfigure}{.47\textwidth}
  \centering
  \includegraphics[width=1\linewidth]{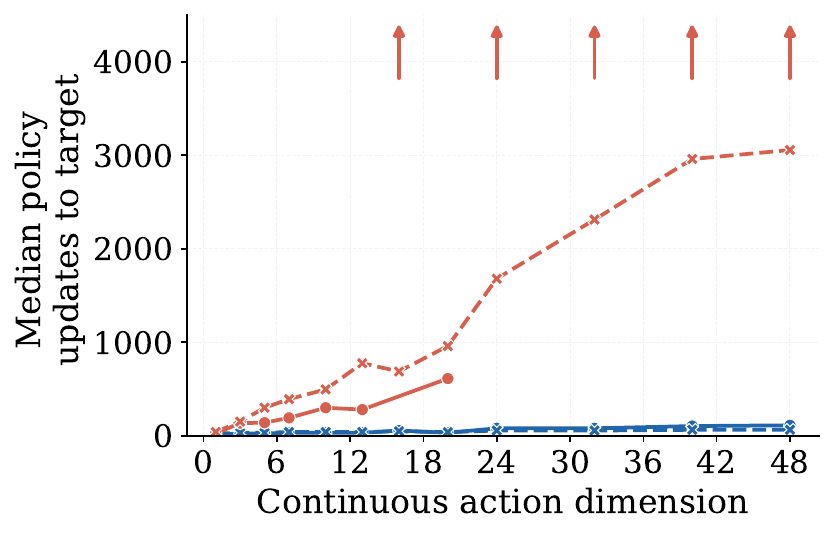}
  \caption{Target gap: $30\%$.}
  \label{fig:scaling-lqr-30}
\end{subfigure}
\caption{Median number of policy updates required to reach a target validation-performance gap
in the switched LQR problem, as the continuous action dimension $p$ varies. Each panel uses a
different target gap. Blue curves correspond to HPO and red curves to PPO. Solid lines use
$J=2$ modes, while dashed lines use $J=8$ modes. Arrows indicate that the median did not
converge for the given configuration.}
\label{fig:scaling-lqr-targets}
\end{figure}
\begin{figure}
\centering
\begin{subfigure}{.47\textwidth}
  \centering
  \includegraphics[width=1\linewidth]{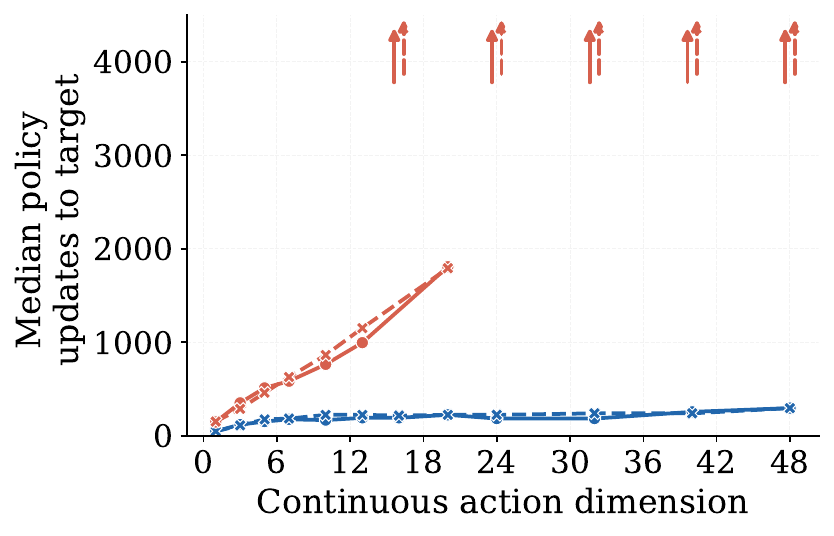}
  \caption{Target gap: $5\%$.}
  \label{fig:scaling-lqr-v2-5}
\end{subfigure}%
\hspace{0.02\textwidth}
\begin{subfigure}{.47\textwidth}
  \centering
  \includegraphics[width=1\linewidth]{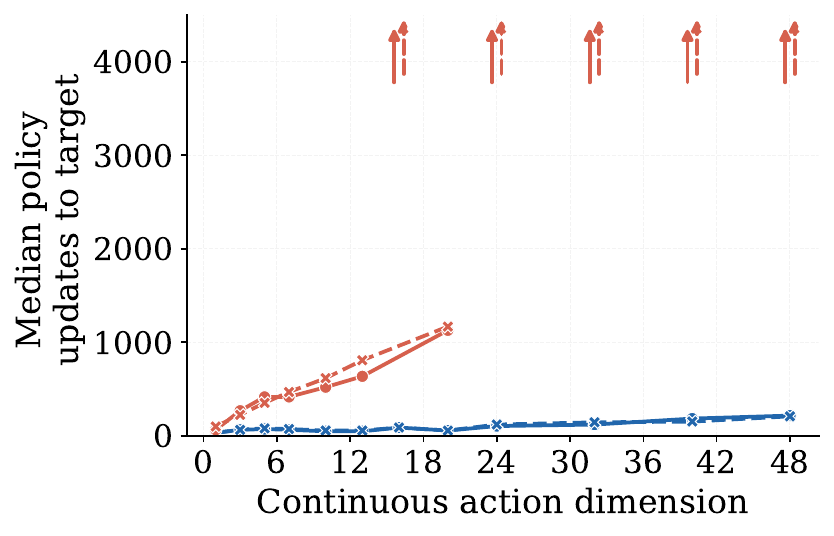}
  \caption{Target gap: $10\%$.}
  \label{fig:scaling-lqr-v2-10}
\end{subfigure}
\vspace{0.5em}
\begin{subfigure}{.47\textwidth}
  \centering
  \includegraphics[width=1\linewidth]{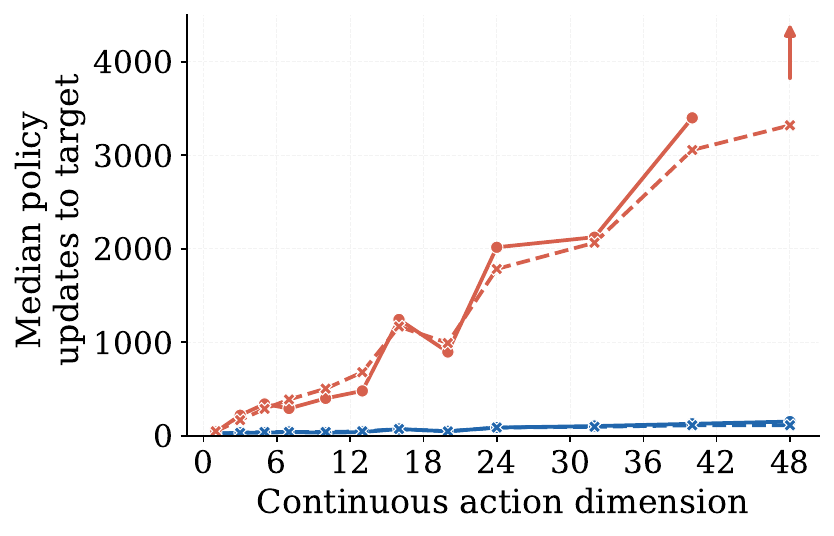}
  \caption{Target gap: $20\%$.}
  \label{fig:scaling-lqr-v2-20}
\end{subfigure}%
\hspace{0.02\textwidth}
\begin{subfigure}{.47\textwidth}
  \centering
  \includegraphics[width=1\linewidth]{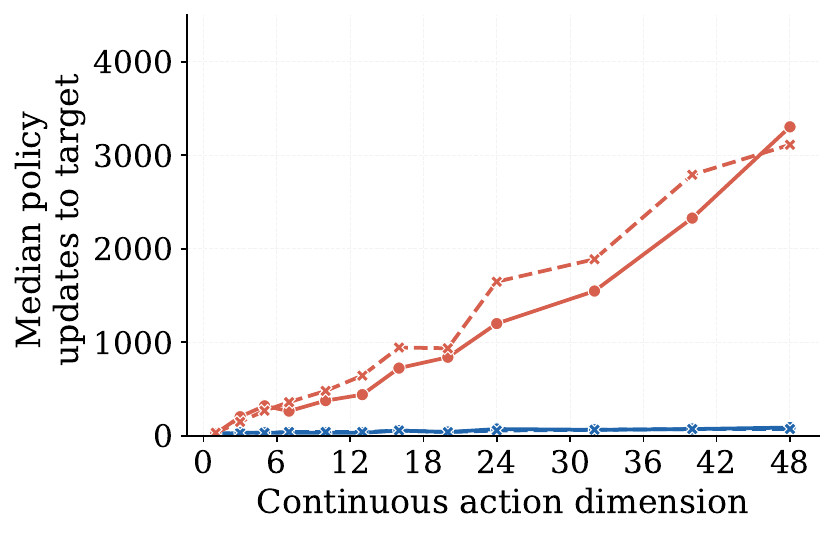}
  \caption{Target gap: $30\%$.}
  \label{fig:scaling-lqr-v2-30}
\end{subfigure}
\caption{Median number of policy updates required to reach a target validation-performance gap
in the switched LQR problem, as the continuous action dimension $p$ varies. Each panel uses a
different target gap. Blue curves correspond to HPO and red curves to PPO. Solid lines use
$J=4$ modes, while dashed lines use $J=6$ modes. Arrows indicate that the median did not
converge for the given configuration.}
\label{fig:scaling-lqr-v2-targets}
\end{figure}

\subsubsection{Performance validation.}
\label{appendix:lqr-performance-validation}
To further validate the quality of the learned policies, we compare HPO against a Riccati-based baseline. For a fixed mode, the Riccati equations yield the optimal linear feedback controller; we therefore consider a baseline that selects a single mode and applies the corresponding Riccati solution. When multiple modes are available, we report the performance of the best such single-mode controller.

We begin with the case $J=1$ for varying values of $p$, corresponding to the standard LQR problem, where the Riccati solution is globally optimal. This provides a ground-truth benchmark against which we can directly assess optimality. As shown in Figure~\ref{fig:hpo-riccati-j1}, HPO achieves near-optimal performance across all dimensions. While a slight degradation is observed as $p$ increases, it remains very modest, with relative gaps below $0.2\%$ even at $p=48$.

We additionally consider the S-LQR setting with $J=4$. In this case, the Riccati baseline corresponds to the best fixed-mode controller, and thus cannot exploit switching. Figure~\ref{fig:hpo-riccati-j4} shows that HPO improves upon this baseline, with gains of up to $13\%$ in some configurations. This indicates that HPO effectively learns to select appropriate modes when beneficial.

\begin{figure}
\centering
\begin{subfigure}{0.48\textwidth}
    \centering
    \includegraphics[width=\linewidth]{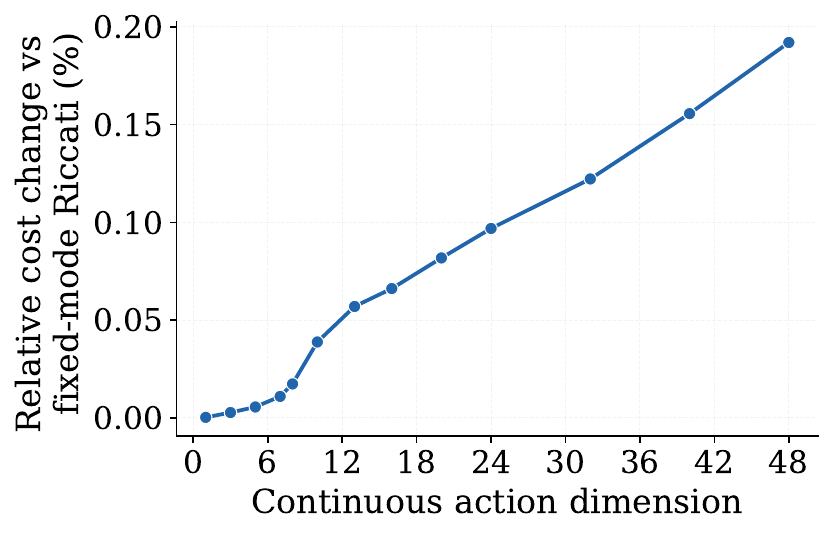}
    \caption{$J=1$ (standard LQR).}
    \label{fig:hpo-riccati-j1}
\end{subfigure}
\hfill
\begin{subfigure}{0.48\textwidth}
    \centering
    \includegraphics[width=\linewidth]{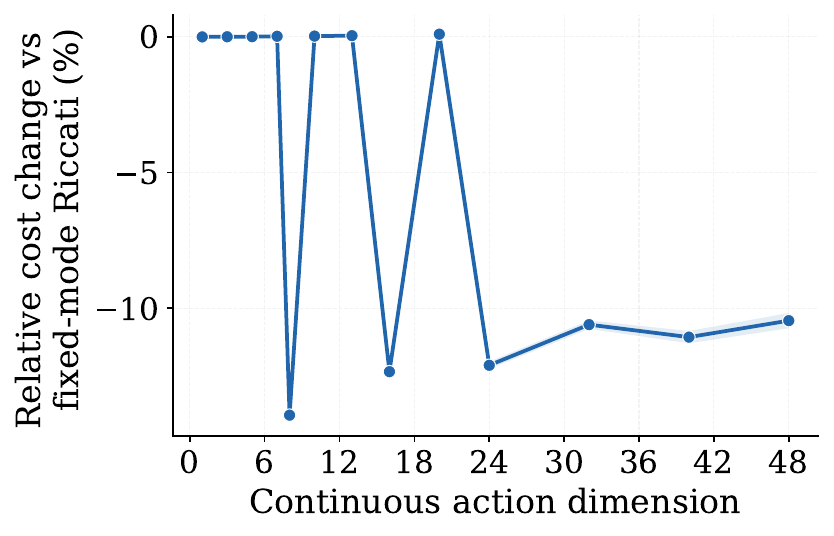}
    \caption{$J=4$ (switched LQR).}
    \label{fig:hpo-riccati-j4}
\end{subfigure}
\caption{Relative cost difference between HPO and Riccati-based baselines as a function of the continuous action dimension $p$. The relative cost difference is computed as the percentage change in cost achieved by HPO compared to a Riccati baseline \textbf{(lower is better)}. For $J=1$, the baseline is the optimal Riccati controller; for $J=4$, the baseline is the best single-mode Riccati controller among the available modes.}
\label{fig:hpo-vs-riccati}
\end{figure}

\subsection{Gradient Quality Experiments}
\label{appendix:gradient-alignment}

This appendix provides a complete specification of the procedure used to evaluate the quality of gradient estimators, as presented in Section \ref{sec:gradient-quality-exps}.

\paragraph{Setting.}
We consider a setting with $p$ identical products, with $p \in \{1,3,5,10,20,30,50\}$. All products have an underage $u=9$, holding cost $h=1$ and i.i.d.\ Poisson demand with mean $10$. Fixed costs are set to $K=64p$. We focus on this simplified setting to obtain more transparent insights. In particular, assuming identical products enables meaningful comparisons across product counts, whereas heterogeneity could introduce confounding effects in our analysis.

\paragraph{Policy construction.}
For each value of $p$, we train $3$ independent policies to near-optimality using \textbf{HPOFull}, with the hyperparameters defined in Appendix \ref{appendix:implementation-details}.
For each trained policy, we construct evaluation policies as follows. We first freeze the parameters of the continuous policy $\pi^{\mathcal{B}}_{\kappa}$ and modify its outputs by applying a constant shift of $2$ units. This avoids PW gradients having zero magnitude, which would render alignment metrics uninformative. We additionally add i.i.d.\ standard Normal noise to each output so that all gradient estimators are evaluated under a common stochastic policy.

We then reinitialize the discrete policy $\pi^{\mathcal{X}}_{\phi}$ and train it together with a value function, while keeping the continuous component fixed. To obtain a reliable critic, we first train the value function alone for several epochs, keeping the discrete policy fixed.
After this initialization phase, we jointly train the discrete policy and the value function. We use a larger learning rate for the value function so that it remains well-calibrated as the discrete policy evolves. We alternate between training epochs and gradient-estimation epochs. During gradient-estimation epochs, both policy components are frozen so that all gradient estimates correspond to the same fixed policy.

\paragraph{Gradient computation.}
We estimate the PW and cross terms using the reparameterization trick~\citep{kingma2013auto,rezende2014stochastic}.

To estimate all SF-type terms (including the discrete gradient and the cross term), we use Monte Carlo rollouts combined with a learned critic baseline, avoiding the bias introduced by PPO-style objectives. Specifically, whenever an action-value $Q_{\pi_{\theta}}(s_t,a_t)$ appears in the gradient expressions~\eqref{eq:grad-phi} and \eqref{eq:grad-kappa}, we replace it with the advantage $Q_{\pi_{\theta}}(s_t,a_t) - V_{\pi_{\theta}}(s_t)$.

The action-value is estimated via sampled cumulative costs, while $V_{\pi_{\theta}}$ is learned by the value network. This corresponds to a REINFORCE-style estimator with a baseline, and yields unbiased gradient estimates with reduced variance relative to vanilla REINFORCE. We refer to this gradient as SF for brevity.

\subsubsection{Batch-Level Gradient Estimators}

Fix a policy $\pi_\theta$. A batch-level gradient estimator with a batch size of $B$ trajectories is defined as
\[
g = \frac{1}{B} \sum_{i=1}^B \hat{g}^{(i)}
\]
where each $\hat{g}^{(i)}$ is a per-trajectory gradient estimate computed from an independent rollout.

We assume $g = \mu + \varepsilon$, where $\mu = \nabla_\theta J(\theta)$ denotes the true gradient and $\varepsilon$ is zero-mean noise with covariance matrix $\Sigma$. In practice, the true gradient $\mu$ is not directly observable, and we approximate it using the iteration-level average mixed gradient, which serves as a low-variance proxy.

All batch estimators are constructed from disjoint sets of trajectories and are therefore conditionally independent given the policy parameters.

\subsubsection{Metric Definitions}

Let $g$ and $h$ be independent batch-level gradient estimators with true gradient $\mu$. For each
metric, we form a Monte Carlo estimate by averaging the relevant quantity over $R$ independent
draws of $(g, h)$, and report the final metric as a function of this average as described below.

\paragraph{Signal.}
The individual-sample estimator $\langle g, h \rangle$ satisfies
$\mathbb{E}[\langle g, h \rangle] = \|\mu\|^2$,
providing an unbiased estimate of the squared gradient norm. We define
\[
    \text{Signal} = \sqrt{\hat{q}^{\mathrm{sig},2}}, \qquad
    \hat{q}^{\mathrm{sig},2} = \frac{1}{R}\sum_{r=1}^R \langle g^{(r)}, h^{(r)} \rangle.
\]

\paragraph{RMSE.}
The individual-sample estimator $\frac{1}{2}\|g - h\|^2$ satisfies
\[
    \mathbb{E}\!\left[\tfrac{1}{2}\|g - h\|^2\right]
    = \mathbb{E}\!\left[\|g - \mu\|^2\right]
    = \mathrm{Tr}(\Sigma),
\]
providing an unbiased estimate of the expected squared error, equal to the trace of the covariance
matrix $\Sigma$. We define
\[
    \text{RMSE} = \sqrt{\hat{q}^{\mathrm{rmse},2}}, \qquad
    \hat{q}^{\mathrm{rmse},2} = \frac{1}{R}\sum_{r=1}^R \frac{1}{2}\|g^{(r)} - h^{(r)}\|^2.
\]

\paragraph{Alignment.}
The individual-sample estimator $\frac{\langle g, h \rangle}{\|g\|\,\|h\|}$ satisfies
$\mathbb{E}\!\left[\frac{\langle g, h \rangle}{\|g\|\,\|h\|}\right] = \mathbb{E}\!\left[\mathrm{cos}(g, h)\right]$,
measuring directional consistency across independent estimates irrespective of magnitude. We define
\[
    \text{Align} = \hat{q}^{\mathrm{align}}, \qquad
    \hat{q}^{\mathrm{align}} = \frac{1}{R}\sum_{r=1}^R
    \frac{\langle g^{(r)}, h^{(r)} \rangle}{\|g^{(r)}\|\,\|h^{(r)}\|}.
\]

\paragraph{Signal-to-noise ratio.}
We define the SNR as the ratio of signal to RMSE:
\[
    \text{SNR} = \frac{\text{Signal}}{\text{RMSE}} =
    \sqrt{\frac{\hat{q}^{\mathrm{sig},2}}{\hat{q}^{\mathrm{rmse},2}}}.
\]
Higher values indicate more reliable gradient estimates, as the gradient signal dominates its
variability.

\subsubsection{Cross-Alignment via Leave-One-Out Estimation}

To evaluate how well a single-batch estimator aligns with the true gradient direction, we construct
a proxy of the true gradient using the iteration-level average gradient.

Let $\{g^{(1)}, \dots, g^{(M)}\}$ be batch-level mixed gradient estimators collected during a
training iteration, and define the iteration-level mean
\[
    G = \frac{1}{M} \sum_{i=1}^M g^{(i)}.
\]
To avoid self-inclusion bias, for a sampled mixed batch $g^{(j)}$ we construct the leave-one-out
estimator
\[
    G_{-j} = \frac{1}{M-1} \sum_{i \neq j} g^{(i)}.
\]
Given a test batch estimator $h$ (corresponding to batch-level estimates of either the mixed or
PW gradient in our experiments), we define
\[
    \text{CrossAlign}(h) = \hat{q}^{\mathrm{calign}}, \qquad
    \hat{q}^{\mathrm{calign}} = \frac{1}{R}\sum_{r=1}^{R}
    \frac{\langle G_{-j^{(r)}}, h^{(j^{(r)})} \rangle}{\|G_{-j^{(r)}}\|\,\|h^{(j^{(r)})}\|},
\]
where $j^{(r)}$ is drawn uniformly at random from $\{1, \dots, M\}$ in each repetition. This metric approximates the alignment between a
batch-level estimator and the true gradient direction.

\subsubsection{Metric estimation \label{appendix:metric-estimation}}

At each training iteration, we collect a set $Q_e$ of batch-level gradient estimators. Metrics are computed using a split-sample procedure.

For alignment, signal, and RMSE, we repeat the following procedure $R = 100$ times: we sample 
two gradients $g, h \in Q_e$ uniformly at random (with replacement) and compute the corresponding 
individual-sample estimator. The final estimate $\hat{q}$ is obtained by averaging over repetitions, 
and the reported metric is computed as a function of $\hat{q}$ as described in the previous subsection.

For cross-alignment, we sample an index $j^{(r)}$ uniformly at random from $\{1, \dots, M=|Q_e|\}$
and compute the leave-one-out quantity above using the paired mixed and PW estimators
corresponding to that index. This procedure is repeated $R = 100$ times and averaged.

\subsubsection{Aggregation}

To compare gradient quality across different training phases and problem sizes, we aggregate 
results according to two variables: the number of products $p$ and the performance gap (relative 
to the best performance achieved after perturbing the continuous policy and retraining the discrete 
policy component). Concretely, for each value of $p$ and each trained policy, we first compute 
the best validation loss achieved by any run of the perturbed policy. We then define the optimality 
gap at a given iteration as the relative difference between the current validation loss and this 
reference value. Performance gaps are discretized into buckets of varying width depending on the 
analysis. Metrics are first averaged within each iteration, then across runs, according to their 
performance gap bucket. Reported values correspond to averages over all runs and iterations for 
each $(p, \text{performance gap bucket})$ pair.

\subsubsection{Results \label{appendix:gradient-quality-results} }

We now present additional results to complement those presented in Section \ref{sec:gradient-quality-exps}. We organize results according to key insights.

\paragraph{Pathwise-only estimator's goodness is driven by signal ratio.}

We now provide guidance about when a PW-only estimator provides a good approximation of the mixed gradient. Our main hypothesis is that performance is governed by the \emph{signal ratio}, defined as the fraction of total squared signal attributable to the PW component; that is, the squared signal of the PW term divided by the sum of the squared signals of the PW and cross components. Intuitively, the PW estimator performs well when most of the gradient signal is carried by the PW term.

To support this, Figure \ref{fig:pathwise-alignment-scatter} plots the signal ratio against its cross-alignment between batch-level PW estimates and iteration-level mixed gradients, across product counts. Each point corresponds to a $(p,\text{performance gap})$ pair, and the dashed line shows a linear fit. We observe a strong positive relationship, with $R^2 = 0.73$, indicating that higher signal ratios directly translate into better alignment with the mixed gradient. Intuitively, when the cross term has comparable or larger magnitude, omitting it significantly distorts the gradient direction.

The signal ratio varies across both problem parameters and training regimes. First, it depends on the instance: larger fixed costs typically reduce the signal ratio, as they increase the magnitude of the cross term (via larger advantages) relative to the PW component. Second, it depends on the discrete policy's quality. Far from optimality, the signal ratio is typically large due to very large PW gradient norms. Additionally, as the discrete policy approaches a best response, the signal ratio approaches $1$ since the cross term's norm vanishes, consistent with Theorem \ref{thm:cross-term-vanishes-for-near-opt}.

Figure \ref{fig:pathwise-alignment-10} illustrates the evolution of both signal ratio and cross-alignment metrics for both PW and mixed batch-level estimators, for $p=10$. PW achieves high CrossAlign across most performance gaps, except in the $20$–$35\%$ range, where a drop in signal ratio coincides with poor alignment. For gaps below $15\%$, PW attains higher CrossAlign than mixed, although alignment degrades overall due to increased variance in the iteration-level mixed gradient. Still, the relatively stronger batch-level alignment suggests that PW remains a better estimator in this regime. This suggests that, if PW-only estimators are to be used, the evolution of the discrete policy must be carefully controlled during training.

\begin{figure}
\begin{subfigure}{.49\textwidth}
  \centering
  \includegraphics[width=1\linewidth]{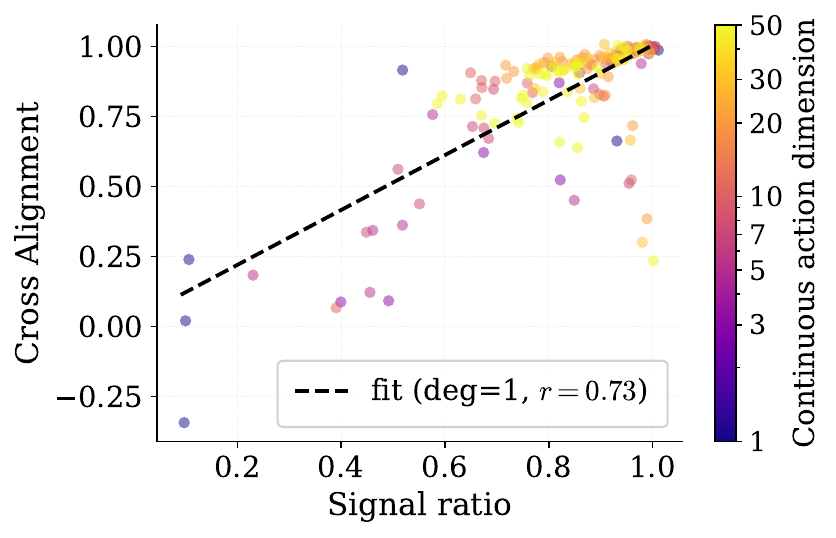}
  \caption{CrossAlign vs.\ signal ratio across all $(p, \text{gap})$ configurations. 
  Color indicates $p$ (darker = smaller dimension). \newline}
  \label{fig:pathwise-alignment-scatter}
\end{subfigure}%
\hspace{0.02\textwidth}
\begin{subfigure}{.49\textwidth}
  \centering
  \includegraphics[width=1\linewidth]{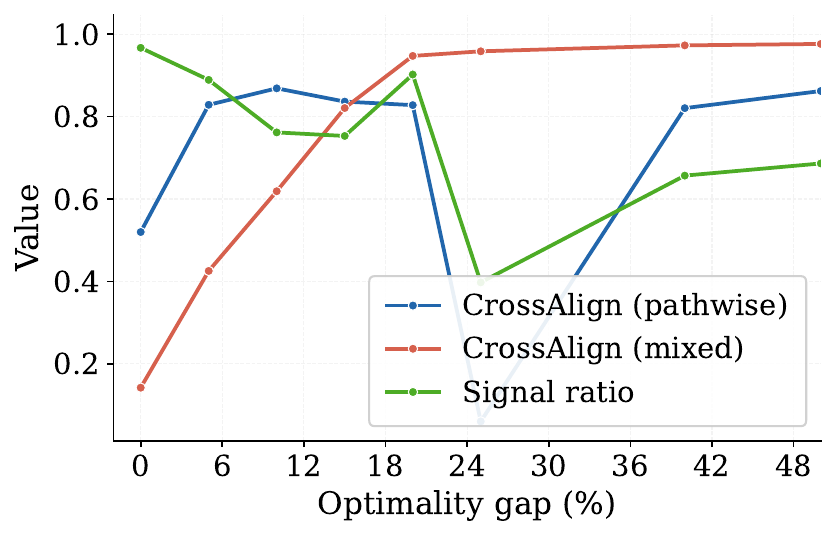}
  \caption{CrossAlign and signal ratio vs.\ optimality gap for $p=10$. Both series share 
  the $[0,1]$ range but measure different quantities: CrossAlign is a cosine similarity, 
  while signal ratio is a relative norm.}
  \label{fig:pathwise-alignment-10}
\end{subfigure}
\caption{Cross alignment between batch-level PW and iteration-level mixed gradient estimates across product counts 
and training regimes.}
\label{fig:pathwise-alignment}
\end{figure}

\paragraph{Cross term degrades close to optimality due to increased sensitivity of discrete policy.}

We analyze the drivers behind the degradation of the cross term near optimality. Figures~\ref{fig:signal-multi-stores} and~\ref{fig:rmse-multi-stores} report the gradient norm and RMSE, respectively, for continuous action dimensions $p=20$ and $p=50$. The RMSE exhibits sharp spikes as the performance gap approaches a best response, occurring precisely in regions where the gradient norm becomes small. This renders the gradient increasingly uninformative near optimality.

\begin{figure}
\centering
\begin{subfigure}{0.48\textwidth}
  \centering
  \includegraphics[width=\linewidth]{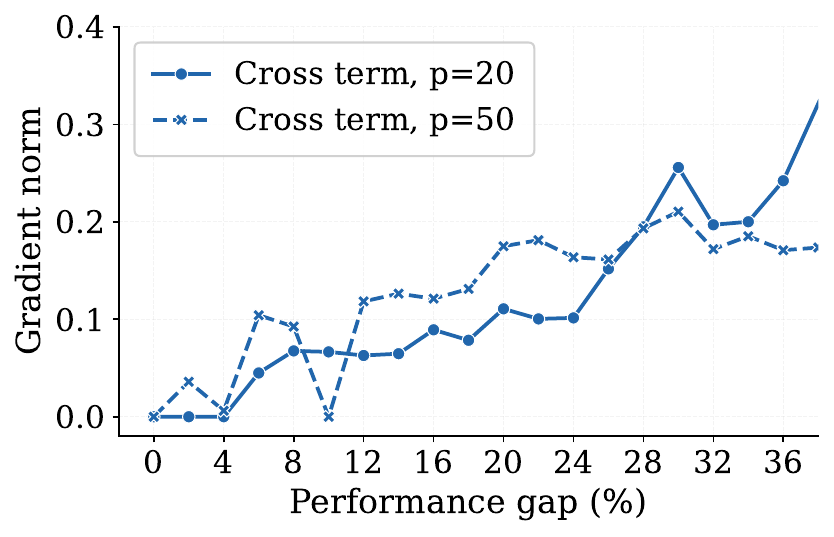}
  \caption{Gradient norm.}
  \label{fig:signal-multi-stores}
\end{subfigure}
\hfill
\begin{subfigure}{0.48\textwidth}
  \centering
  \includegraphics[width=\linewidth]{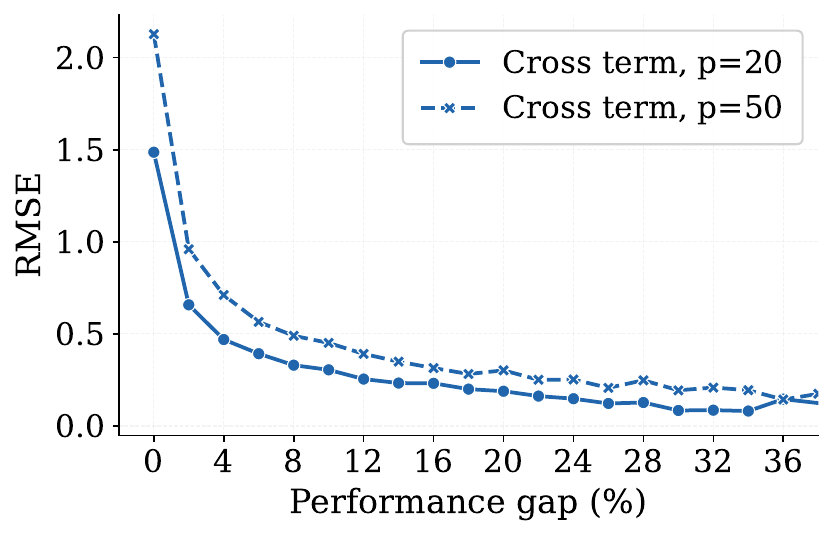}
  \caption{RMSE.}
  \label{fig:rmse-multi-stores}
\end{subfigure}

\caption{
Gradient norm and RMSE of batch-level gradient estimates of the cross term in the JRP setting for $p=20$ and $p=50$.
}
\label{fig:gradient-multi-stores}
\end{figure}

To understand this, Figure~\ref{fig:discrete-policy-evolution} plots the probability of not ordering as a function of total system on-hand and in-transit inventory across different performance gaps in the JRP with one product. As shown in \cite{scarf1960optimality}, under linear underage and holding costs with shared fixed costs and i.i.d.\ demand, the optimal policy is characterized by an $(s, S)$ policy: there exists a number $s$ such that if total inventory falls on or below $s$, we should place an order that increases total inventory to $S$. We observe that the learned policy seems to approximate such a thresholded structure, becoming increasingly sharp as the performance gap decreases.

This implies that the sensitivity of the discrete policy to the state becomes highly uneven around the threshold. The policy is nearly flat far from the threshold but becomes very steep in its vicinity. As a result, the SF term $\nabla_\kappa \log \pi^{\mathcal{X}}_\phi(x_t \mid s_t(\kappa))$ exhibits high variability across trajectories, depending on whether the state lies near this transition region. This variability drives the increase in the variance of the cross term, even as its expectation decreases.

\begin{figure}
\centering

\begin{subfigure}{0.48\textwidth}
    \centering
    \includegraphics[width=\linewidth, trim=0 0 0 0, clip]{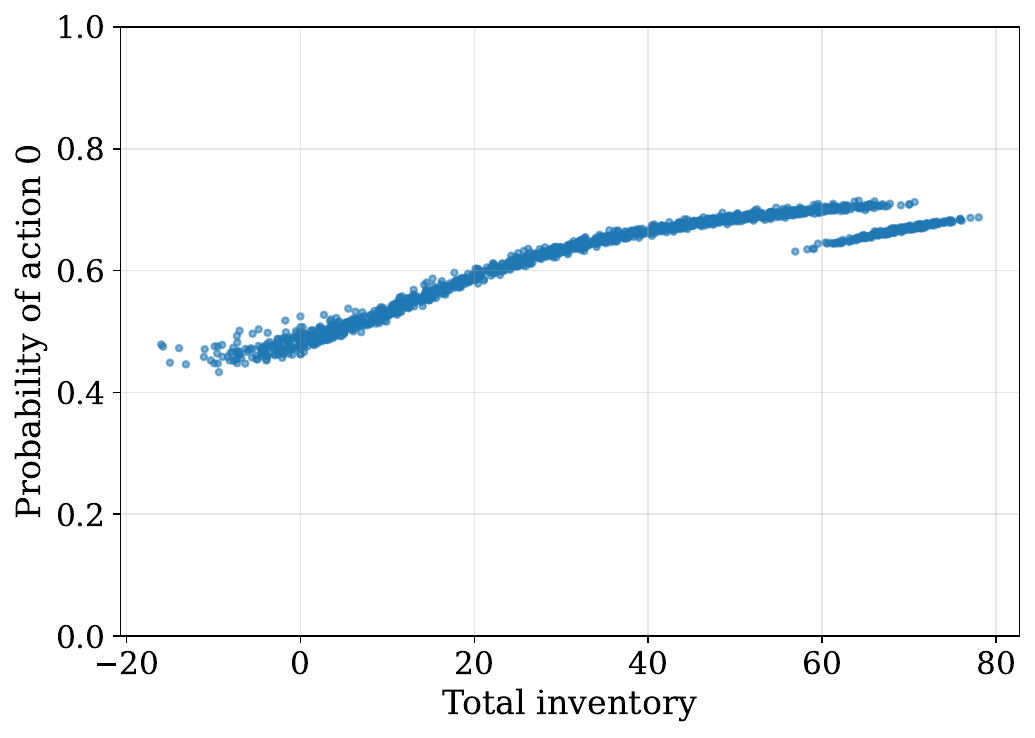}
    \caption{Gap $\approx 68\%$}
\end{subfigure}
\hfill
\begin{subfigure}{0.48\textwidth}
    \centering
    \includegraphics[width=\linewidth, trim=0 0 0 0, clip]{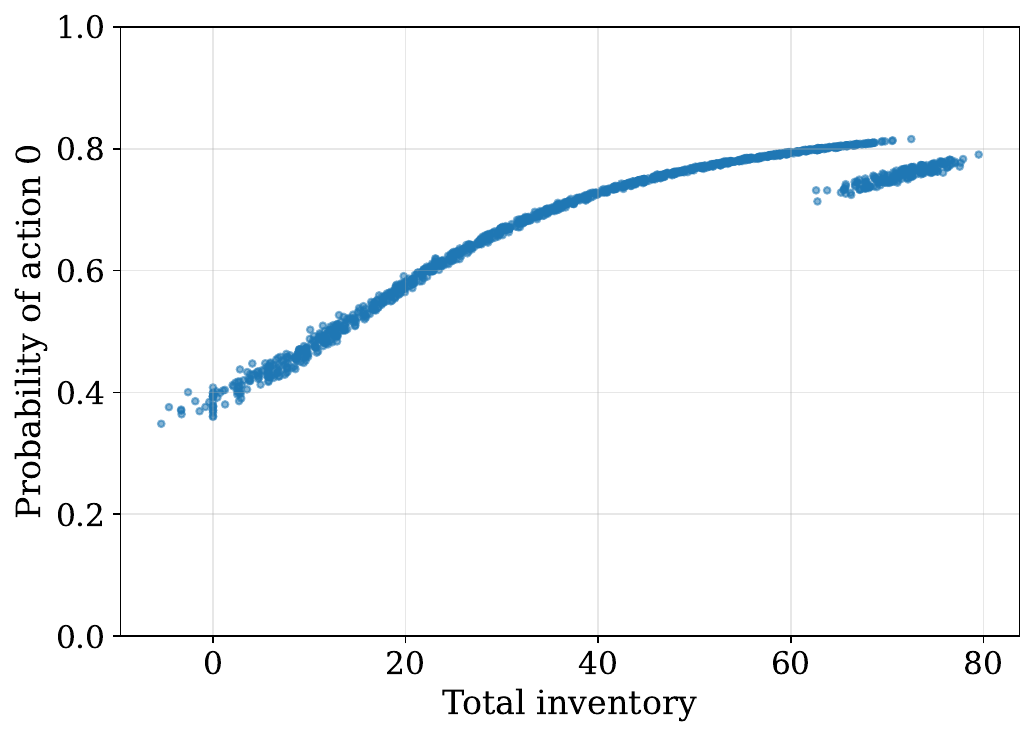}
    \caption{Gap $\approx 26\%$}
\end{subfigure}

\vspace{0.5em}

\begin{subfigure}{0.48\textwidth}
    \centering
    \includegraphics[width=\linewidth, trim=0 0 0 0, clip]{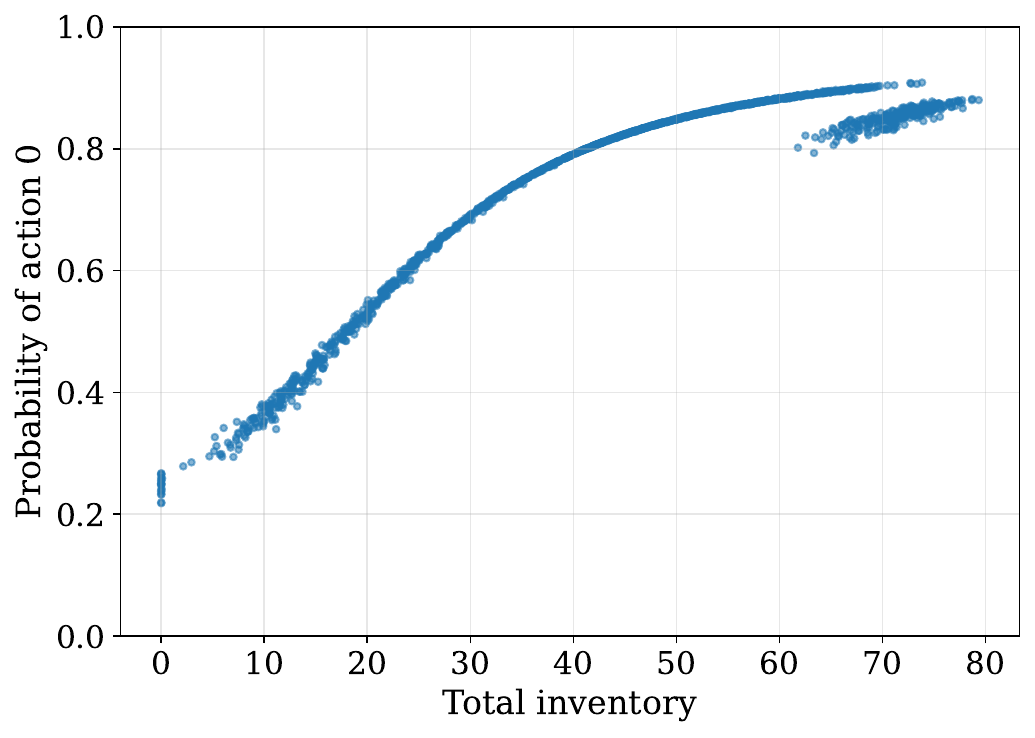}
    \caption{Gap $\approx 10\%$}
\end{subfigure}
\hfill
\begin{subfigure}{0.48\textwidth}
    \centering
    \includegraphics[width=\linewidth, trim=0 0 0 0, clip]{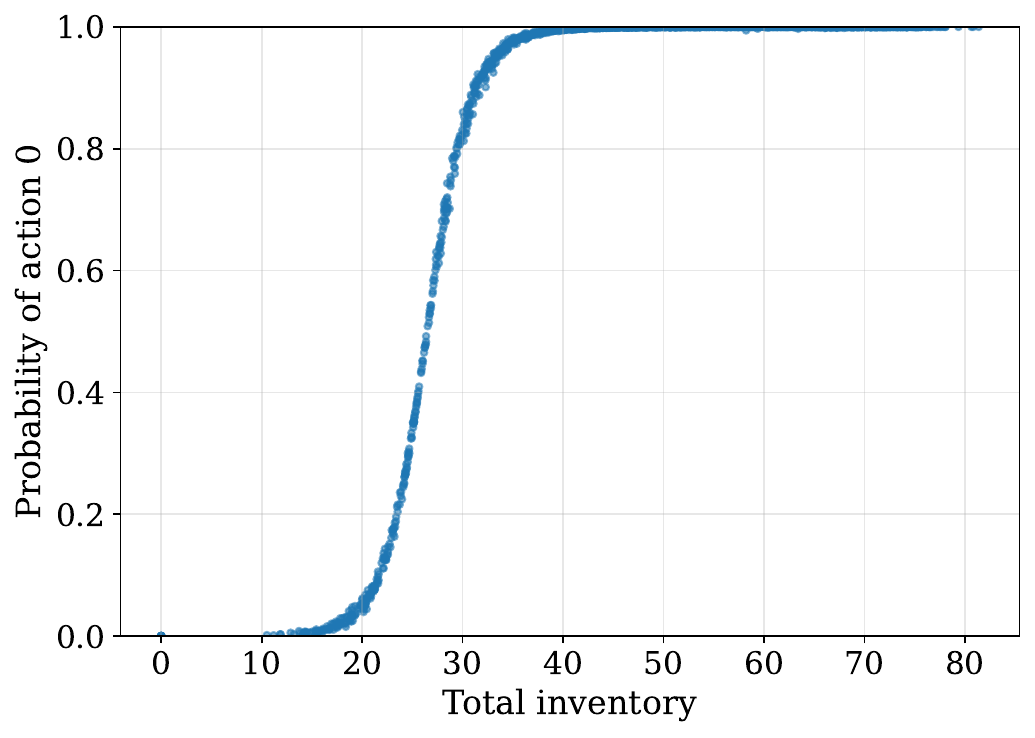}
    \caption{Gap $\approx 0\%$}
\end{subfigure}

\caption{Probability of not placing an order as a function of total system inventory, for several optimality gaps.}
\label{fig:discrete-policy-evolution}
\end{figure}

\subsection{PPO Hyperparameter Robustness}
\label{appendix:ppo-hyperparams}

Figure~\ref{fig:scaling-jrp-one-off} reports a hyperparameter ablation study for PPO, comparing one-off variants against the baseline PPO and HPO configurations (PPO~1 and HPO~1, respectively). The baseline PPO~1 uses learning rate $10^{-4}$, entropy coefficient $\beta_{\text{ent}} = 0.5$ and $5$ optimization epochs per iteration (i.e., each collected batch is reused for $5$ gradient passes) in which both the discrete and continuous policies are updated. Each one-off variant isolates a single change to one of these knobs: PPO~2 lowers entropy regularization to $\beta_{\text{ent}} = 0.01$; PPO~3 increases the number of PPO epochs per iteration from $5$ to $10$; PPO~4 enables discrete-only extra epochs at the base learning rate $10^{-4}$; PPO~5 increases the learning rate to $10^{-3}$ without extra discrete epochs; and PPO~6 combines both ($10^{-3}$ with discrete-only extra epochs). The discrete-only extra epoch variants (PPO~4 and PPO~6) deserve special attention: in the base setup, 5 gradient steps are taken with respect to both discrete and continuous parameters simultaneously. Enabling discrete-only extra epochs instead performs $5$ full gradient passes for the discrete parameters but only a single pass for the continuous ones. The motivation is to mirror how HPO operates, where discrete parameters take 5 gradient steps per batch but the continuous only take 1, thereby testing whether a comparable update structure can close the gap between PPO and HPO. An additional variant with $1$ optimization epoch per batch was included in the sweep but is omitted from the figure, as it failed to satisfy the convergence criterion for any value of $p$.

Across all variants, baseline PPO~1 is the strongest PPO configuration. All variants converge more slowly than PPO~1 for most values of $p$, with several exhibiting erratic behavior at small $p$ before diverging steeply. Notably, even the discrete-only extra epoch variants, which were specifically designed to mimic HPO's update structure, fail to match PPO~1 let alone HPO~1. In contrast, HPO~1 outperforms every PPO variant across the full range of $p$, with the gap widening consistently as $p$ grows. These results suggest that HPO's advantage over PPO is robust to hyperparameter tuning: even with careful one-off adjustments, no PPO variant approaches HPO's convergence speed at larger action dimensions.

\begin{figure}
\centering
\begin{subfigure}{\textwidth}
  \centering
  \includegraphics[width=0.9\linewidth]{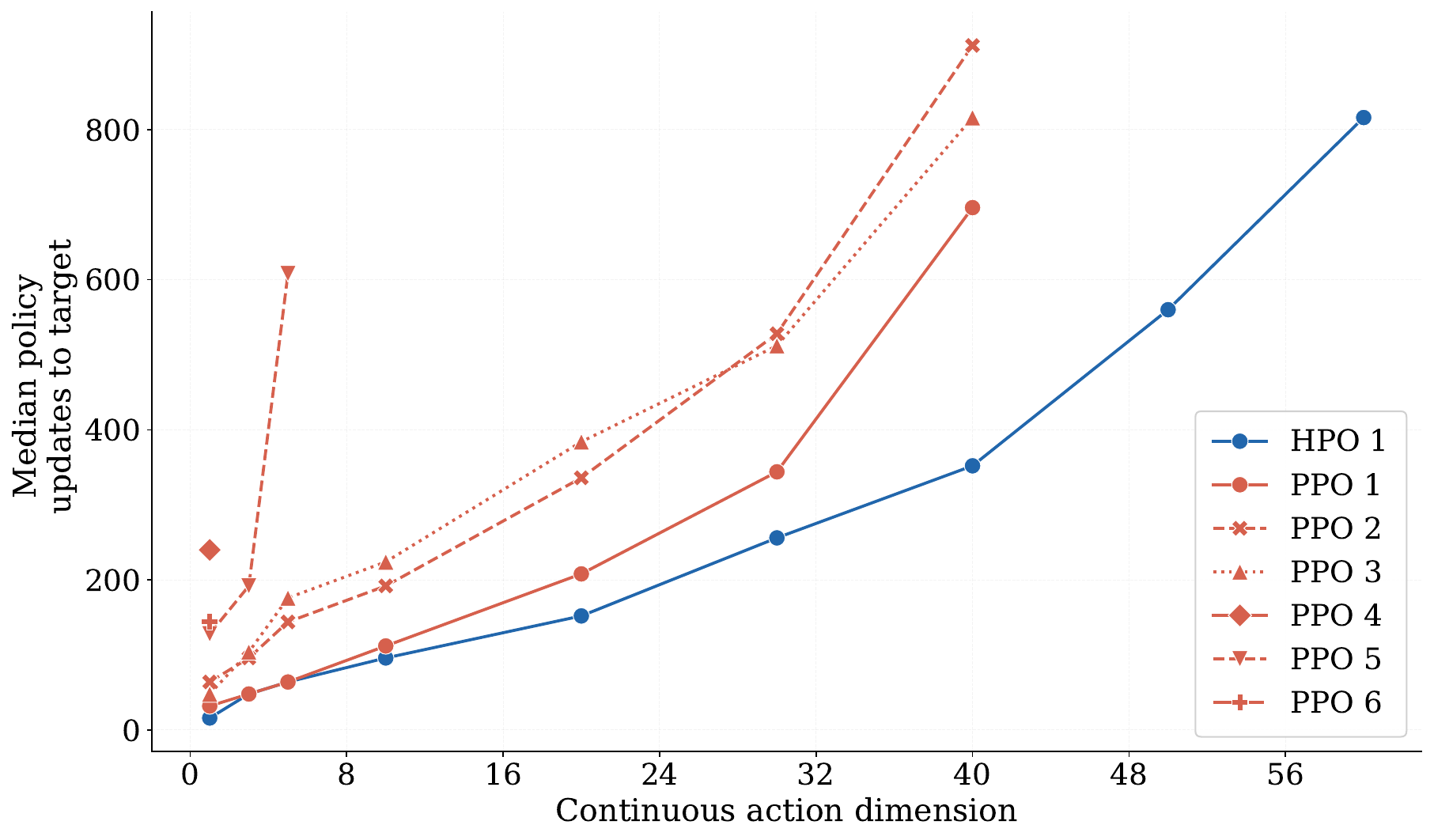}
\end{subfigure}
  \caption{Median policy updates to reach a $10\%$ target gap for HPO and six PPO 
hyperparameter configurations, as the continuous action dimension $p$ varies. 
Configurations for which the median did not converge within the training budget are 
omitted.}
  \label{fig:scaling-jrp-one-off}
\end{figure}

\end{document}